\documentclass{article}
\usepackage[preprint]{neurips2020}

\usepackage[utf8]{inputenc} 
\usepackage[T1]{fontenc}    
\usepackage{hyperref}       
\usepackage{url}            
\usepackage{booktabs}       
\usepackage{amsfonts}       
\usepackage{nicefrac}       
\usepackage{microtype}      

\usepackage{bm}
\usepackage{url}
\usepackage{float}
\usepackage{amsmath}
\usepackage{amssymb}
\usepackage{amsthm}
\usepackage{xspace}
\usepackage{color}
\usepackage{graphicx}
\usepackage{grffile}
\usepackage{soul}
\usepackage{etoolbox}
\usepackage{chngpage}
\usepackage{environ}
\usepackage{subcaption}
\usepackage{rotating}
\usepackage{multirow}
\usepackage{algorithm}
\usepackage{algorithmicx}
\usepackage[noend]{algpseudocode}
\usepackage{cleveref}
\usepackage{bigints}
\usepackage{mathtools}
\usepackage{wrapfig}

\def\eqref#1{equation~\ref{#1}}

\DeclarePairedDelimiter{\ceil}{\lceil}{\rceil}

\def\1{\bm{1}}

\DeclareMathAlphabet{\mathsfit}{\encodingdefault}{\sfdefault}{m}{sl}
\SetMathAlphabet{\mathsfit}{bold}{\encodingdefault}{\sfdefault}{bx}{n}

\DeclareMathOperator*{\argmin}{arg\,min}

\DeclarePairedDelimiter{\norm}{\lVert}{\rVert}

\DeclareMathOperator{\sign}{sign}

\DeclarePairedDelimiterX{\inp}[2]{\langle}{\rangle}{#1, #2}

\DeclareMathOperator\erf{erf}
\newtoggle{shortpaper}

\newtheorem{prop}{Proposition}

\newtheorem{example}{Example}

\begin{document}

\title{Unique properties of adversarially trained linear classifiers on Gaussian data}

\author{Jamie Hayes\\
University College London\\
{\tt\small j.hayes@cs.ucl.ac.uk}
}


\maketitle
\thispagestyle{empty}

\begin{abstract}

Machine learning models are vulnerable to adversarial perturbations, that when added to an input, can cause high confidence misclassifications.
The adversarial learning research community has made remarkable progress in the understanding of the root causes of adversarial perturbations.
However, most problems that one may consider important to solve for the deployment of machine learning in safety critical tasks involve high dimensional complex manifolds that are difficult to characterize and study. 
It is common to develop adversarially robust learning theory on simple problems, in the hope that insights will transfer to `real world datasets'.
In this work, we discuss a setting where this approach fails. 
In particular, we show with a linear classifier, it is always possible to solve a binary classification problem on Gaussian data under arbitrary levels of adversarial corruption during training, and that this property is not observed with non-linear classifiers on the CIFAR-10 dataset. 

\end{abstract}


\section{Introduction \& background}
\label{sec: introduction}

In the standard nomenclature of supervised machine learning, we let $f_{\theta}: \mathcal{X} \rightarrow \mathbb{R}^k$, represent a hypothesis, where $k$ is the number of classes in the classification problem, and $\theta$ are parameters of the hypothesis function.
We associate this hypothesis with a loss function $\ell : \mathbb{R}^k \times \mathcal{Y} \rightarrow \mathbb{R}$, that takes the output of $f_{\theta}$ as an input along with the true class, and outputs a `cost' associated with the prediction. 
Given a training set $\{(x^1, y^1), (x^2, y^2), \dots, (x^n, y^n)\}$, where $(x^i, y^i)$ are independently drawn from the joint distribution, $(\mathcal{X}, \mathcal{Y})$, \emph{standard training} through empirical risk minimization states we should choose a hypothesis that solves:

\begin{align}
    \argmin_{\theta} \frac{1}{n}\sum^n_{i=1}\ell(f_{\theta}(x^i), y^i) \label{eq: erm}
\end{align}

By choosing a hypothesis that solves \cref{eq: erm}, since the training set was drawn independently from the joint distribution $(\mathcal{X}, \mathcal{Y})$, we expect a small loss on other samples that will be drawn from this distribution in the future, and consequently, the learned hypothesis to perform well on future data.
Adversarial training~\citep{madry2017towards} has been proposed as a method to boost the performance of a classifier against adversarial examples~\citep{szegedy2013intriguing, biggio2018wild}, by choosing a hypothesis that solves:

\begin{align}
    \argmin_{\theta} \frac{1}{n}\sum^n_{i=1}\max_{\norm{\delta^i}\leq \epsilon}\ell(f_{\theta}(x^i + \delta^i), y^i) \label{eq: adv_erm}
\end{align}

By choosing a hypothesis that solves \cref{eq: adv_erm}, we expect, for $x\in \mathcal{X}$, the classification of $f_{\theta}(x)$ is equal to $f_{\theta}(x + \delta)$, providing $\delta$ is small with respect to a norm, $\norm{\cdot}$. 
Throughout this work, we consider perturbations, $\delta$, with an $\ell_{\infty}$ norm smaller than $\epsilon$.

The adversarial learning research community has expended a great deal of effort to understand generalization properties of machine learning models learnt with adversarial training. 
By studying simple data models, one can understand how adversarial training affects the (standard and robust) generalization error of a hypothesis. In particular, works in this space have studied the following simple data and hypothesis models: 

\vspace{0.5cm}

\noindent \textbf{Data model.} Let $(x,y)\in \mathbb{R}^d \times \{\pm 1\}$, where $y\in\{\pm 1\}$ is sampled uniformly at random, and $x\mid y \sim \mathcal{N}(y\mu, \sigma^2 I)$, where  $\mu\in\mathbb{R}_{+}^d$ and $\sigma \in \mathbb{R}_{+}$.

\vspace{0.5cm}

\noindent \textbf{Hypothesis model.} We consider the linear classifier $F_{\theta}(x) = \sign(f_{\theta}(x))$, where $f_{\theta}(x) = \inp{\theta}{x}$ and $\theta\in\mathbb{R}^d$.

\vspace{0.5cm}

\citet{schmidt2018adversarially} show that under this simple data and hypothesis model, adversarial training requires $\Omega(d)$ number of training samples, while standard training needs only a constant number of training samples to reach a small generalization error.
\citet{tsipras2018robustness} show that in a similar setting a trade-off between the test set
accuracy of a model on standard inputs and its robustness to adversarial perturbations provably exists, and demonstrated this phenomenon can be observed in more complex data and hypothesis models.
\citet{nakkiran2019adversarial} went on to prove that robustness may come at the expense of simplicity, by demonstrating under a simple data model, there exists no linear classifier that is robust to adversarial examples.

\citet{chen2020more} proved that under this data and hypothesis model, there exists cases where adversarial training can cause the generalization gap between an adversarially trained classifier and a standard classifier to increase and subsequently decrease with more training data, and cases where more training data increases this gap. 
The exact behavior of the generalization gap as the size of training set grows is dependent on the strength of the adversary (also referred to as the adversaries \emph{perturbation budget}) used in the inner-maximization in \cref{eq: adv_erm} -- that is, the choice of $\epsilon$. 
\citet{min2020curious} showed under the same setting, depending on the choice of $\epsilon$, as the number of training inputs increases, the generalization error can either (1) monotonically decrease, (2) form a double descent curve where the error decreases, increases and then decreases again, or (3) decrease and then monotonically increase. 
\citet{chen2020more} and~\citet{min2020curious} proved these properties using a linear loss function, that has also been studied in relation to other robustness properties by~\citet{khim2018adversarial} and~\citet{yin2018rademacher}, and is defined by:

\begin{align}
\label{eq: linear_loss}
\begin{split}
    \ell(f_{\theta}(x), y) &= -yf_{\theta}(x) \\
    &= -y\inp{\theta}{x} 
\end{split}
\end{align}

\noindent \textbf{Contributions.} 
In this work, we discuss why it may be insufficient to consider such simple data and hypothesis models in the quest to understand how adversarial training affects the generalization error on `real-world datasets'. 
We prove that in this simple setting, a linear classifier with a linear loss can learn a hypothesis with a standard generalization error equal to the Bayes error rate, under adversarial training with an arbitrarily large perturbation budget, $\epsilon$. 
Additionally, we give empirical evidence this result holds for other types of loss function.
Since the perturbation budget can be arbitrarily large, this implies that the classifier can learn under seemingly completely mislabeled data, a property we show is not exhibited on an adversarially trained model on the CIFAR-10 dataset~\citep{krizhevsky2009learning}. 
We argue for caution when studying adversarial training and related robustness properties under these simple settings. It is not necessarily the case that they are informative of more general properties of adversarial robustness that one may desire to understand.

\section{Further background}
\label{sec: background}

In this section, in addition to the aforementioned related work introduced in~\cref{sec: introduction}, we briefly survey some of the most relevant works on the generalization of adversarially robust models.

Similar to work by~\citet{tsipras2018robustness} and~\citet{nakkiran2019adversarial}, \citet{papernot2016towards} formally studied the trade-off between robustness, simplicity and accuracy, finding that adversarial example vulnerability can manifest due to fundamental limitations in the expressivity of the hypothesis class, and can be resolved by moving to a richer hypothesis class.
The generalization error under adversarial perturbations within an $\epsilon$-ball (referred to as the \emph{robust error}), is at least as large as standard test set error.
Using this observation, \citet{zhang2019theoretically} studied the trade-off between robustness and accuracy in a simple data model setting by decomposing the robust error into the sum of the standard generalization error and the \emph{boundary error}. 
The boundary error corresponds to how likely input features are within an $\epsilon$ distance to the decision boundary. 
By minimizing a differentiable upper bound, they develop a new adversarial example defense that scales to large datasets such as tiny ImageNet.

\citet{dohmatob2018generalized} generalizes the results of \citet{tsipras2018there} to a richer set of distributions, showing that high test set accuracy is not sufficient to prove the absence of adversarial examples.
\citet{raghunathan2019adversarial} further explore these trade-offs by showing adversarial training hurts generalization even when the optimal predictor has both optimal standard and robust accuracy, and show incorporating unlabeled data in training, as introduced by \citep{carmon2019unlabeled, najafi2019robustness, uesato2019are}, reduces adversarial example vulnerability.
\citet{raghunathan2020understanding} continue this investigation, developing a theoretical characterization of the trade-off between standard and robust error in linear regression that motivates the incorporation of unlabeled data in training as a method to improve robust error without sacrificing standard error.

Other works explore adversarial robustness through the lens of randomization~\citep{pinot2019theoretical}, computational hardness~\citep{bubeck2018adversarial, gourdeau2019hardness, mahloujifar2018can}, statistical and PAC learning~\citep{diochnos2019lower, cullina2018pac, montasser2019vc}, and optimal transport theory~\citep{bhagoji2019lower, li2018optimal, pydi2019adversarial}.

Importantly, our results do not contradict works on the fundamental robustness-accuracy trade-off. 
Our work shows it is possible to achieve perfect test set accuracy under any perturbation budget used in adversarial training.
We make no comment on the robustness of the final classifier under the data model introduced in~\cref{sec: introduction}. 
Typically, robustness-accuracy trade-offs are explored under the assumption that the adversarial perturbation budget is small in comparison to the possible space of feature values.
In our work, the semantic definition of an adversarial example is blurred, since during training, the perturbation budget is large enough such that the perturbed dataset is equivalent to a mislabeling of each data point.

\section{A linear classifier has small generalization error if $\theta$ is positive}
\label{sec: problem1}

We begin by showing that, under the previously introduced data and hypothesis model, if the learnt parameter $\theta$ is positive, then the linear classifier has a generalization error equal to the Bayes error rate. 
We first give an intuitive analysis of why in one dimension and the proceed to generalize to the higher dimensional case.

\begin{figure}[t]
  \centering
  \includegraphics[width=0.75\linewidth]{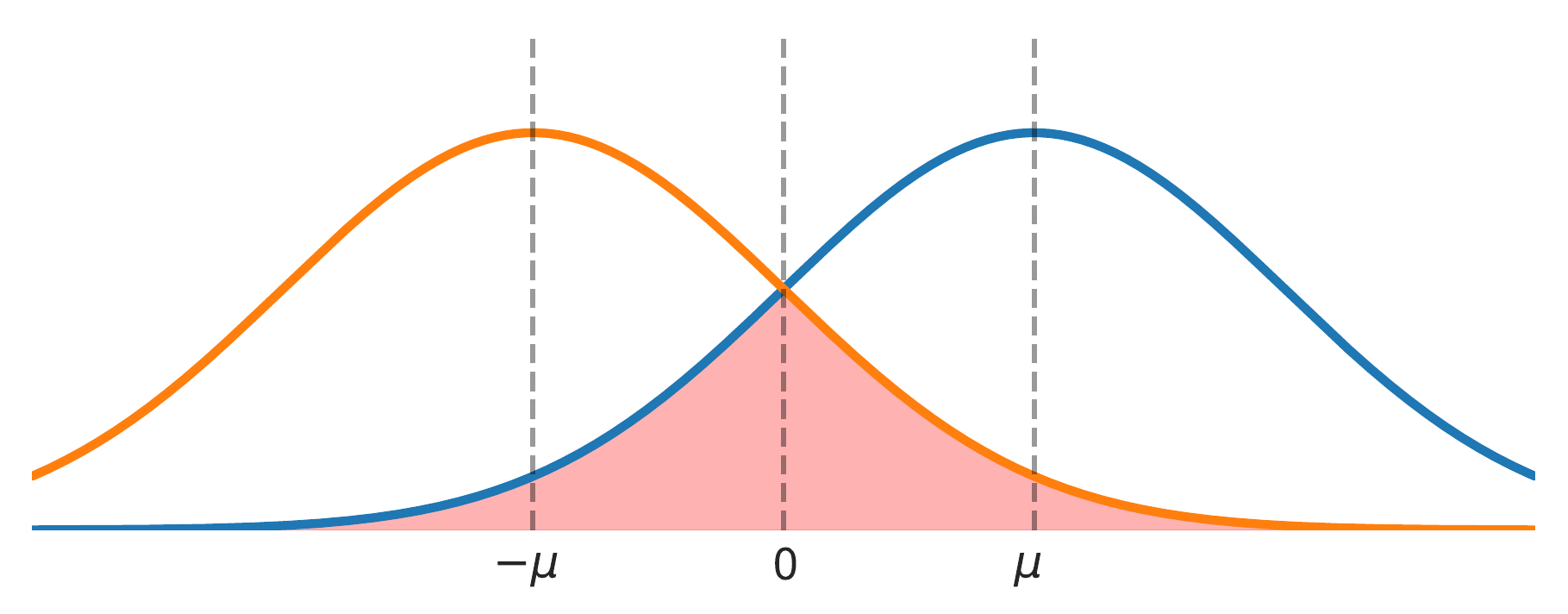}
  \caption{Illustration of a one dimensional Gaussian data model introduced in~\cref{sec: introduction}. The red shaded region under the curve corresponds to the prediction errors incurred by the Bayes classifier.}
  \label{fig:gaussian_1d}
\end{figure}

\Cref{fig:gaussian_1d} gives an illustration of the data model defined in \cref{sec: introduction} for $d=1$. 
The Bayes error is given by the red shaded region, and clearly the optimum decision boundary is centered at $\frac{\mu + (-\mu)}{2}=0$.
The Bayes error rate of this data model can be written as:
\begin{align}
    \label{eq: bayes_error}
    \begin{split}
        P_{\text{error}} &= \frac{1}{2}\int^0_{-\infty}P(x|y=1)dx + \frac{1}{2}\int^{\infty}_{0}P(x|y=-1)dx \\
        &= \frac{1}{2}\int^0_{-\infty} \frac{1}{\sqrt{2\pi}\sigma}e^{-\frac{(x-\mu)^2}{2\sigma^2}}dx + \frac{1}{2}\int^{\infty}_{0}\frac{1}{\sqrt{2\pi}\sigma}e^{-\frac{(x+\mu)^2}{2\sigma^2}}dx \\
        &= \frac{1}{2}\big(1 - \erf(\frac{\mu}{\sqrt{2}\sigma})\big)
    \end{split}
\end{align}

Since the optimal decision boundary is at zero, the linear classifier will achieve the minimal possible generalization error if $f_{\theta}(x)$ does not flip the sign of the input, which is equivalent to constraining $\theta$ to be positive. 
The signed linear classifier achieves the Bayes error rate if $\sign(\theta) = 1$. 

We now show the zero vector is the optimal decision boundary in $d>1$ dimensions, and that as $d$ grows, if $\theta>0$, the probability that $\sign(\theta^Tx)\neq\sign(x)$ shrinks to zero. 
Thus the Bayes error tends to zero and it is achieved when $\theta >0$.

\vspace{0.5cm}

\begin{prop}
\label{prop:d_dimension_bayes}

If $\theta>0$, then we achieve the Bayes error rate in $d$ dimensions.

\end{prop}

\begin{proof}
Let $x\in\mathbb{R}^d$, so if $y=1$, $x$ is sampled from a Gaussian with pdf equal to $\frac{1}{(2\pi)^{\frac{d}{2}}\sigma}\exp{(-\frac{1}{2}(x-\mu)^{T}\sigma^2 I(x-\mu) )}$.
Similarly if $y=-1$, $x$ is sampled from a Gaussian with pdf equal to $\frac{1}{(2\pi)^{\frac{d}{2}}\sigma}\exp{(-\frac{1}{2}(x+\mu)^{T}\sigma^2 I(x+\mu) )}$.

The Bayes decision rule is found at:

\begin{align}
    &P(y=1\mid x) = P(y=-1\mid x) \\
    \implies &\log(\frac{P(y=1\mid x)}{P(y=-1\mid x)}) =0 \\
    \implies &\log(\frac{P(x\mid y=1)}{P(x\mid y=-1)}) =0 \qquad\qquad \text{ since } P(y=1)=P(y=-1) \\
    \implies &\frac{2}{\sigma^2}\mu^{T}x \gtrless 0 
\end{align}

Note that $\frac{2}{\sigma^2}\mu^{T}x > 0 \iff x>0$ and $\frac{2}{\sigma^2}\mu^{T}x < 0 \iff x<0$, since $\frac{2}{\sigma^2}\mu^{T}>0$.
Clearly then, we again have a decision rule that is optimal if the sign of $x$ is not flipped. 
That is, for $\theta\in\mathbb{R}^d$, $\sign(\sum_{j=1}^d \theta_j x_j) = \sign(\sum_{j=1}^d x_j)$.

In the case where $\forall j\in [d]$, $\mu_j = \mu$, we have $\sum_{j=1}^d x_j \sim \mathcal{N}(dy\mu, d\sigma^2I)$.
Without loss of generality, let $y=1$, then $P(\sum_{j=1}^d x_j >0)=\Phi(\frac{\sqrt{d}\mu}{\sigma})$. 
As $d\rightarrow \infty$, $\Phi(\frac{\sqrt{d}\mu}{\sigma})\rightarrow 1$, and so $P(\sum_{j=1}^d \theta_jx_j > 0)\rightarrow1$ if $\theta > 0$. 
Clearly then any $\theta >0$ gives the Bayes optimal classifier in $d$ dimensions and the Bayes error approaches zero as $d$ increases.

\end{proof}

In sum, to show a linear classifier achieves the smallest generalization error under adversarial training, we must show it is more likely that the learned parameter $\theta_j$ is positive, $\forall j\in[d]$, under \emph{any} adversarial perturbation budget, $\epsilon$.

\section{A linear classifier can have perfect standard test set accuracy under adversarial training with any $\epsilon > 0$}
\label{sec: problem1}

We show here that it is more likely that a linear classifier with a linear loss optimizing~\cref{eq: adv_erm} via gradient descent will learn positive parameters, $\theta$, for a sufficiently large number of training steps.

We first note that for the linear loss, the inner-maximization term in~\cref{eq: adv_erm} can be solved exactly: $\max_{\norm{\delta}_{\infty} \leq \epsilon}\ell(f_{\theta}(x + \delta), y) = -y\inp{\theta}{x} + \epsilon\norm{\theta}_1$. 
For $j\in[d]$, denote the value of $j^{\text{th}}$ index of $\theta$, at the $k^{\text{th}}$ iteration of gradient descent, as $\theta^k_j$. 
We have $\frac{\partial \ell}{\partial \theta_j} = -yx_j + \epsilon\sign(\theta_j)$, and so for a given learning rate, $\eta$, we have the following relation:

\begin{align}
\theta^{k+1}_j &= \theta^{k}_j + \eta\big(yx_j - \epsilon\sign(\theta^k_j)\big) 
\end{align}

Note that, in expectation, $\mathbb{E}_{x_j\sim\mathcal{N}(y\mu_j, \sigma^2)}[yx_j] = \mu_j$, and thus, on average, gradient descent obeys the following relation:

\begin{align}
\theta^{k+1}_j = \begin{cases}
    \theta^k_j + \eta(\mu_j + \epsilon) , & \text{if } \theta^k_j < 0 \\
    \theta^k_j + \eta\mu_j  , & \text{if } \theta^k_j = 0 \\
    \theta^k_j + \eta(\mu_j - \epsilon) , & \text{if } \theta^k_j > 0 
  \end{cases}
  \label{eq: theta_iter}
\end{align}

Throughout this work, we consider an initial $\theta^1_j\sim \mathcal{N}(0, \sigma^2I)$, and so

\begin{align}
\sign(\theta^1_j) = \begin{cases}
    -1  , & \text{w.p } \nicefrac{1}{2} \\
    +1  , & \text{w.p } \nicefrac{1}{2}
  \end{cases}
\end{align}

Firstly,  if $0<\epsilon < \mu_j$, then $\lim_{k\rightarrow \infty} \theta^k_j = \infty$. Thus, the linear classifier found by optimizing a linear loss approaches the Bayes error rate if, during adversarial training, the adversary cannot, in expectation, change the sign of an input. 
Before we show that it is more likely that a linear classifier optimizing~\cref{eq: adv_erm} can achieve a generalization error equal to the Bayes error rate under any adversarial perturbation budget, we motivate our findings by way of example:

\vspace{0.5cm}

\begin{example}
Let $\theta^k_j = r$, where $r$ is positive but negligible, $\eta=\frac{1}{2}$, $\epsilon=\frac{3}{2}\mu$, $\mu = 1$. 

Then $\theta^k_j = r \rightarrow \theta^{k+1}_j = r - \frac{1}{4} \rightarrow \theta^{k+2}_j = r + 1 \rightarrow \theta^{k+3}_j = r + \frac{3}{4} \rightarrow \theta^{k+4}_j = r + \frac{1}{2} \rightarrow \theta^{k+5}_j = r + \frac{1}{4} \rightarrow \theta^{k+6}_j = r$. 
Thus, only $\{\theta^{k+1}_j\}$ is negative while $\{\theta^k_j, \theta^{k+2}_j, \theta^{k+3}_j, \theta^{k+4}_j, \theta^{k+5}_j\}$ are all positive. If we train for $m$ steps, where $m$ is a large integer chosen at random, it is five times more likely that the final learned parameter, $\theta^m_j$, is positive. Thus, it is more likely gradient descent will output a classifier that achieves small generalization error.
\end{example}

We now formalize the above example. Firstly, by proving that if $\epsilon > \mu_j$, then zero is an attraction point in \cref{eq: theta_iter}. 

\vspace{0.5cm}

\begin{prop}
 Let $\epsilon > \mu_j  $, then $\forall n \in \mathbb{N}$, $\exists n_0 \geq n$ such that $\theta^{n_0}_j < 0$.
\end{prop}

\begin{proof}
Without loss of generality, let $\theta^n_j > 0$, otherwise we can take $n_0$ to be $n$. 
By \cref{eq: theta_iter}, the sequence $\theta^n_j \rightarrow \theta^{n+1}_j \rightarrow \theta^{n+2}_j \rightarrow \dotsb$ is decreasing until, for some $m\in \mathbb{N}$, $\theta^{n+m}_j < 0$ or $\theta^{n+m}_j = 0$. 
If $\theta^{n+m}_j < 0$ we take $n_0$ to be $n+m$.
Otherwise, if $\theta^{n+m}_j = 0$, then $\theta^{n+m+1}_j = \eta\mu_j$ and $\theta^{n+m+2}_j = \eta(2\mu_j - \epsilon)$. 
Similarly, if $\epsilon > 2\mu_j$, then $\theta^{n+m+2}_j < 0$. 

Now, we focus on the case where $\theta^{n+m+2}_j = \eta(2\mu_j - \epsilon) > 0$. 
Observe that, for $s\in\mathbb{N}$, under the assumption that we have not observed a negative value before $\theta^{n+m+s}_j$, then $\theta^{n+m+s}_j = \eta(s\mu_j - (s-1)\epsilon)$. 
Now, $\theta^{n+m+s}_j = \eta(s\mu_j - (s-1)\epsilon) >0$ when $\epsilon <\frac{s}{s-1}\mu_j$. 
However, in the limit, $\lim_{s\rightarrow\infty} \frac{s}{s-1}\mu_j = \mu_j$, and given that we assume $\mu_j < \epsilon$, there exists $s_0$, such that  $\theta^{n+m+s_0}_j < 0$.
\end{proof}

Using an almost identical argument, it is easy to show $\forall n \in \mathbb{N}$, $\exists n_0 \geq n$ such that $\theta^{n_0}_j > 0$, and so $\theta^k_j$ oscillates around zero as $k\rightarrow \infty$. 
However, the next proposition illustrates that, providing $\epsilon > \mu_j$, as the number of iterations of gradient descent increases, it decreases in probability that the learned parameter, $\theta_j$, is negative. 
This is because after a sufficient number of iterations, whenever the learned parameters turns negative, since $\epsilon > \mu_j$, the parameter value at the following iteration is guaranteed to be positive. After this, we show the converse doesn't hold; for any $\epsilon > \mu_j$ and a sufficiently large number of training steps, it does not always hold that the next parameter in gradient descent is negative if the preceding parameter is positive.

\vspace{0.5cm}

\begin{prop}
\label{prop:next_is_pos}
Assuming $\mu_j < \epsilon$, let $n$ denote the minimum index where $\theta^n_j < 0$ and $n > n_0$, where $n_0$ is the first index satisfying $\theta^{n_0}_j > 0$. That is, if $\theta^1_j >0$, then $n_0=1$ and we consider the smallest $n>n_0$ satisfying $\theta^n_j<0$. 
Otherwise, if $\theta^1_j <0$, then let $n_0$ be the minimum index satisfying $\theta^{n_0}_j > 0$, and so $n$ is the minimum index greater than $n_0$ satisfying $\theta^n_j<0$. Then, $\theta^{n+1}_j > 0$, and $\forall k \in \mathbb{N}$, $\theta^{n+k}_j < 0 \implies \theta^{n+k+1}_j > 0$.
\end{prop}

\begin{proof}
$\forall k \in \mathbb{N}$ where $\theta^{n+k}_j < 0$, we have the lower bound $\eta(\mu_j - \epsilon)<\theta^{n+k}_j$. 
Then, $2\eta\mu_j<\theta^{n+k+1}_j$.
\end{proof}

The following proposition shows that, providing $\mu_j < \epsilon < t\mu_j$ for $t>1$, during gradient descent there exists consecutive positive parameter values. Taking~\cref{prop:next_is_pos} and~\cref{prop:consecutive_pos} together, for a sufficiently large number of training steps it is more likely that gradient descent ends with a learnt positive parameter than a negative parameter, which in turn implies a classifier that is accurate on unseen data from the Gaussian data model.

\vspace{0.5cm}

\begin{prop}
\label{prop:consecutive_pos}
Let $n$ denote the minimum index satisfying $\theta^n_j < 0$ and $n > n_0$, where $n_0$ is the first index satisfying $\theta^{n_0}_j > 0$. 
If $\mu_j < \epsilon < t\mu_j$, where $t>1$, then there $\exists k\in \mathbb{N}$ with $k>n$ such that $\theta^{k}_j > 0$ and $\theta^{k+1}_j > 0$.
\end{prop}

\begin{proof}
We note that $\theta^n_j$ is lower bounded by $\eta(\mu_j - \epsilon)$.
Let us assume there is no $k \in \mathbb{N}$, with $k>n$, such that $\theta^{k}_j > 0$ and $\theta^{k+1}_j > 0$. 
Then $\{\theta^n_j, \theta^{n+1}_j, \theta^{n+2}_j, \theta^{n+3}_j, \theta^{n+4}_j, \dots\}$ is lower bounded by the sequence $\{\eta(\mu_j - \epsilon), 2\eta\mu_j, \eta(3\mu_j - \epsilon), 4\eta\mu_j, \eta(5\mu_j - \epsilon), \dots\}$. 
In particular, $\forall s\in \mathbb{N}$ with $s \equiv 1 \mod 2$, we have $\theta^{n+s}_j > \eta s\mu_j$, and $\theta^{n+s+1}_j > \eta((s+1)\mu_j - \epsilon)$. 
By assumption, $\theta^{n+s+1}_j < 0$, however, taking 

\begin{align}
\label{eq: s_choice}
s = \begin{cases}
    \ceil{t} , & \text{if } \ceil{t} \equiv 0 \mod 2 \\
    \ceil{t} + 1 , & \text{if } \ceil{t} \equiv 1 \mod 2
  \end{cases}
\end{align}

we have $\theta^{n+s}_j > 0$ and $\theta^{n+s+1}_j > \eta((s+1)\mu_j - \epsilon) > 0$, since $\epsilon < t\mu_j < (s+1)\mu_j$.
\end{proof}

\section{Experiments}
\label{sec: experiments}

\begin{figure*}[t]
  \centering
  \begin{subfigure}{0.49\textwidth}
  \centering
  \includegraphics[width=\linewidth]{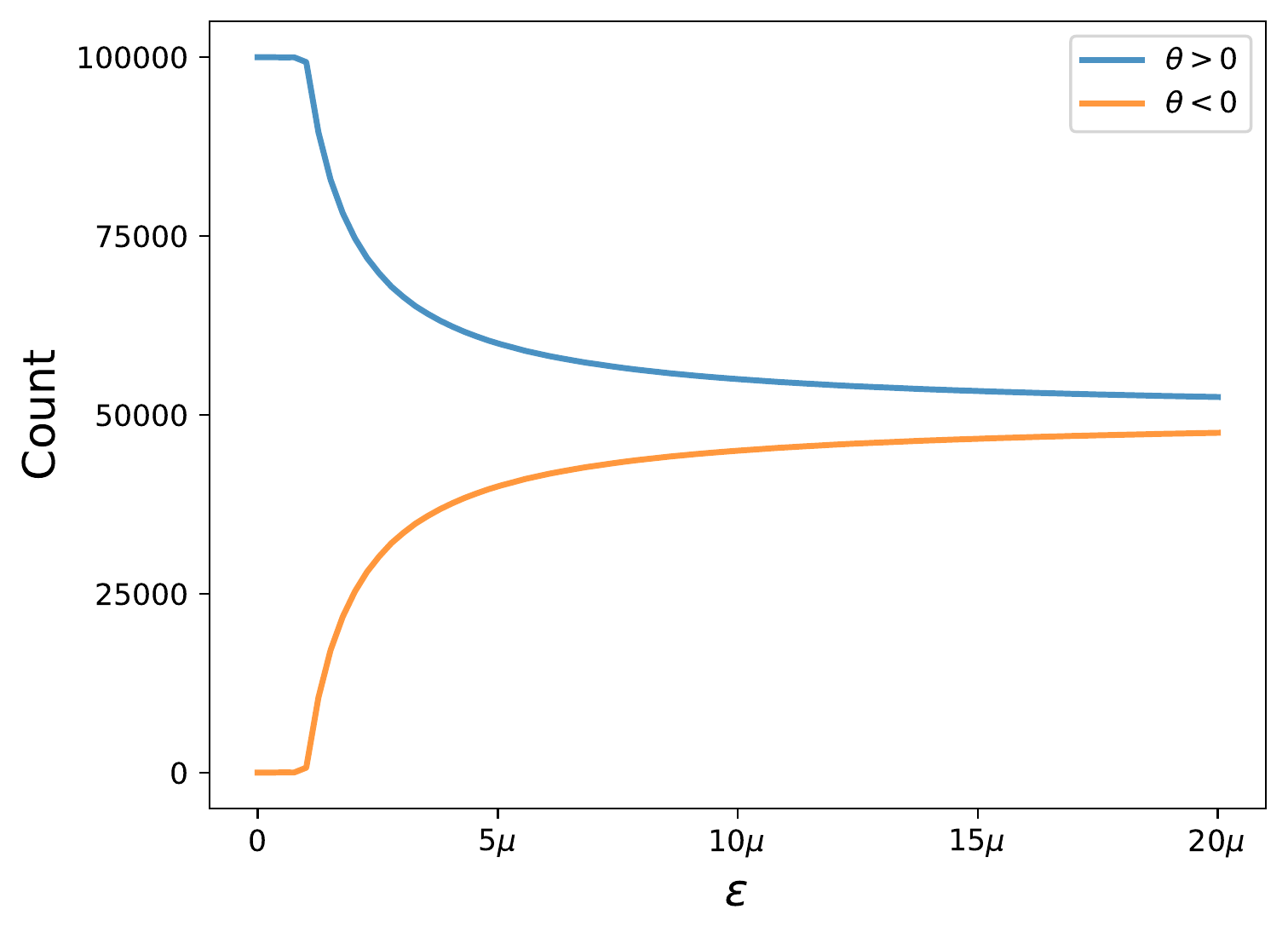}
  \caption{Linear loss}
  \label{fig:1d_linear_loss}
  \end{subfigure}%
  \begin{subfigure}{0.49\textwidth}
  \centering
  \includegraphics[width=\linewidth]{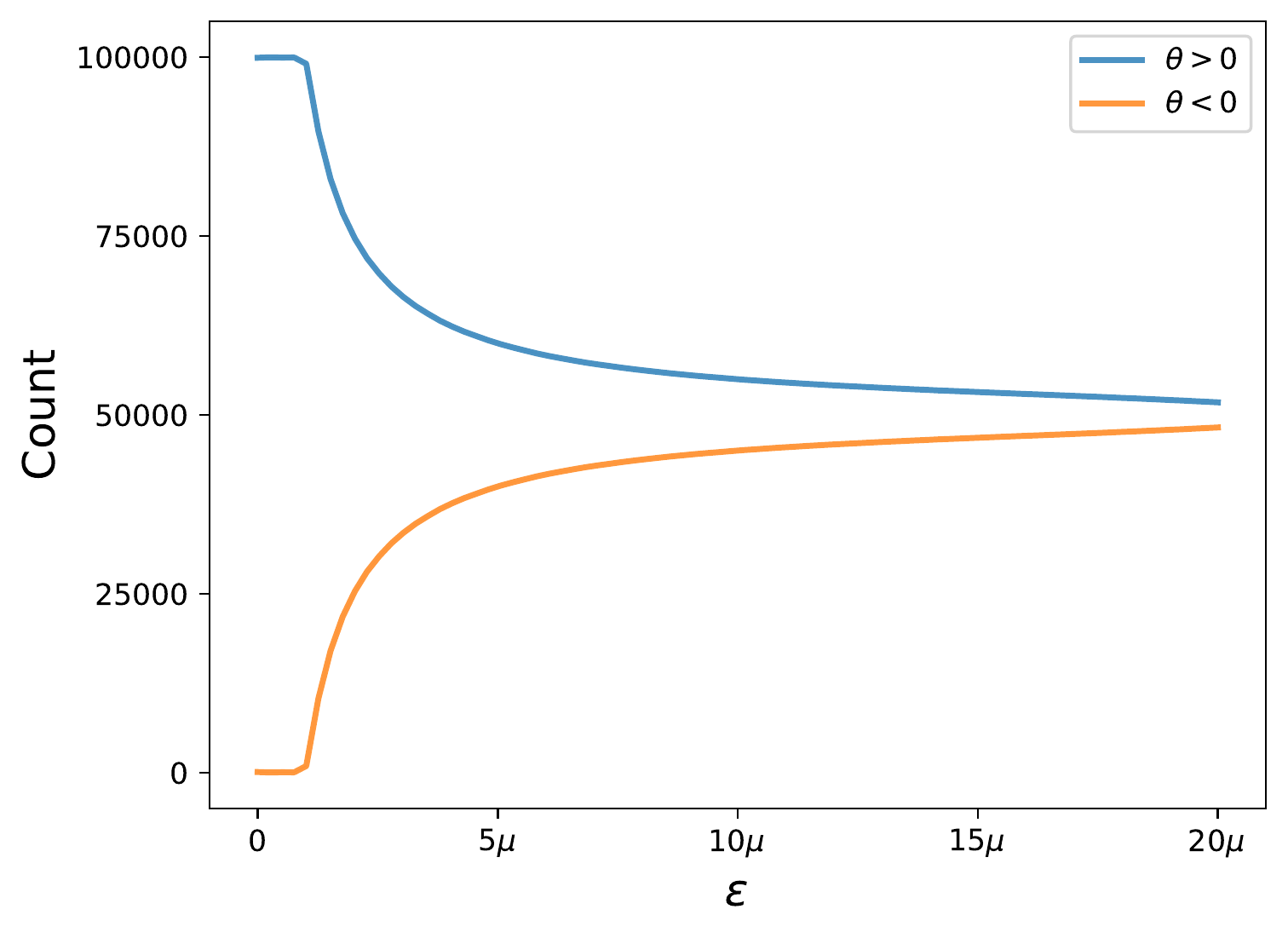}
  \caption{Cross-entropy loss}
  \label{fig:1d_cross_entropy_loss}
  \end{subfigure}
  \caption{The number of iterations during gradient descent, referred to as counts, where $\theta >0$ and where $\theta < 0$, for $0<\epsilon<20$, with $\mu=1$, using either the linear or cross-entropy loss function.}
  \label{fig:1d_loss}
\end{figure*}

\begin{figure*}[t]
\centering
\begin{subfigure}{0.33\textwidth}
  \centering
  \includegraphics[width=\linewidth]{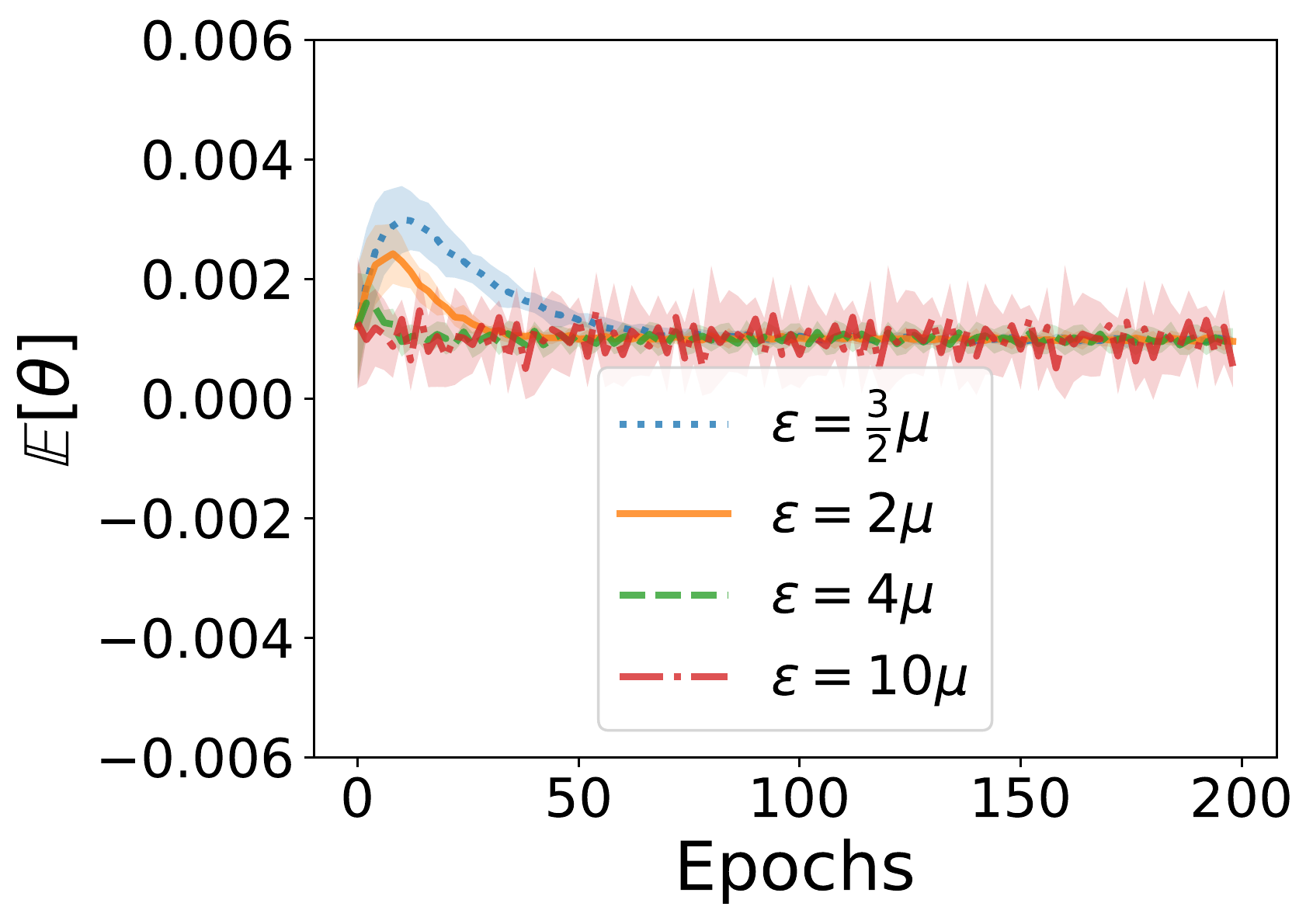}
  \caption{Linear loss, average $\theta$.}
  \label{fig:100d_linear_loss_avg_w_high}
\end{subfigure}%
\begin{subfigure}{0.33\textwidth}
  \centering
  \includegraphics[width=\linewidth]{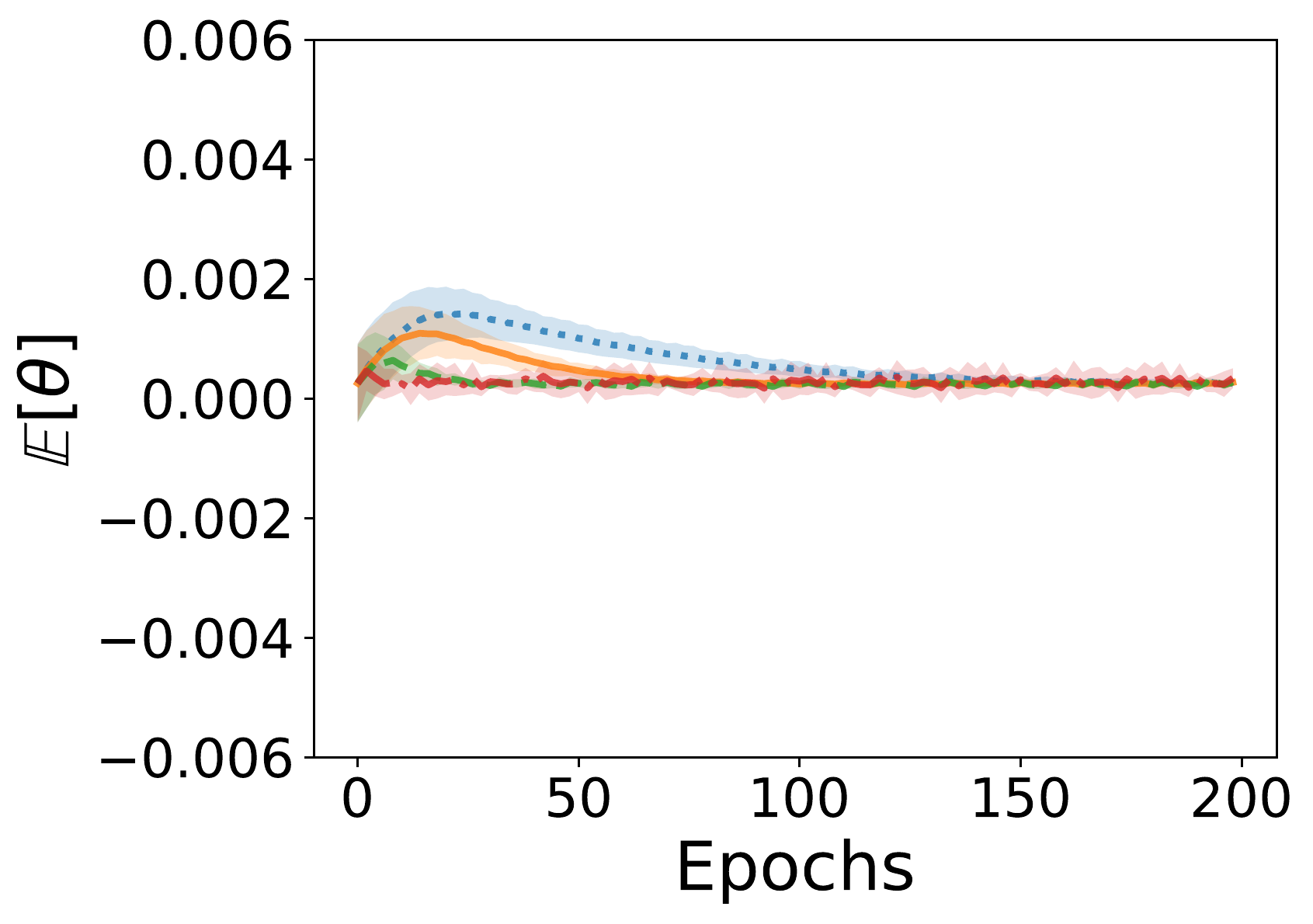}
  \caption{Cross-entropy loss, average $\theta$.}
  \label{fig:100d_cross_entropy_loss_avg_w_high}
\end{subfigure}%
\begin{subfigure}{0.33\textwidth}
  \centering
  \includegraphics[width=\linewidth]{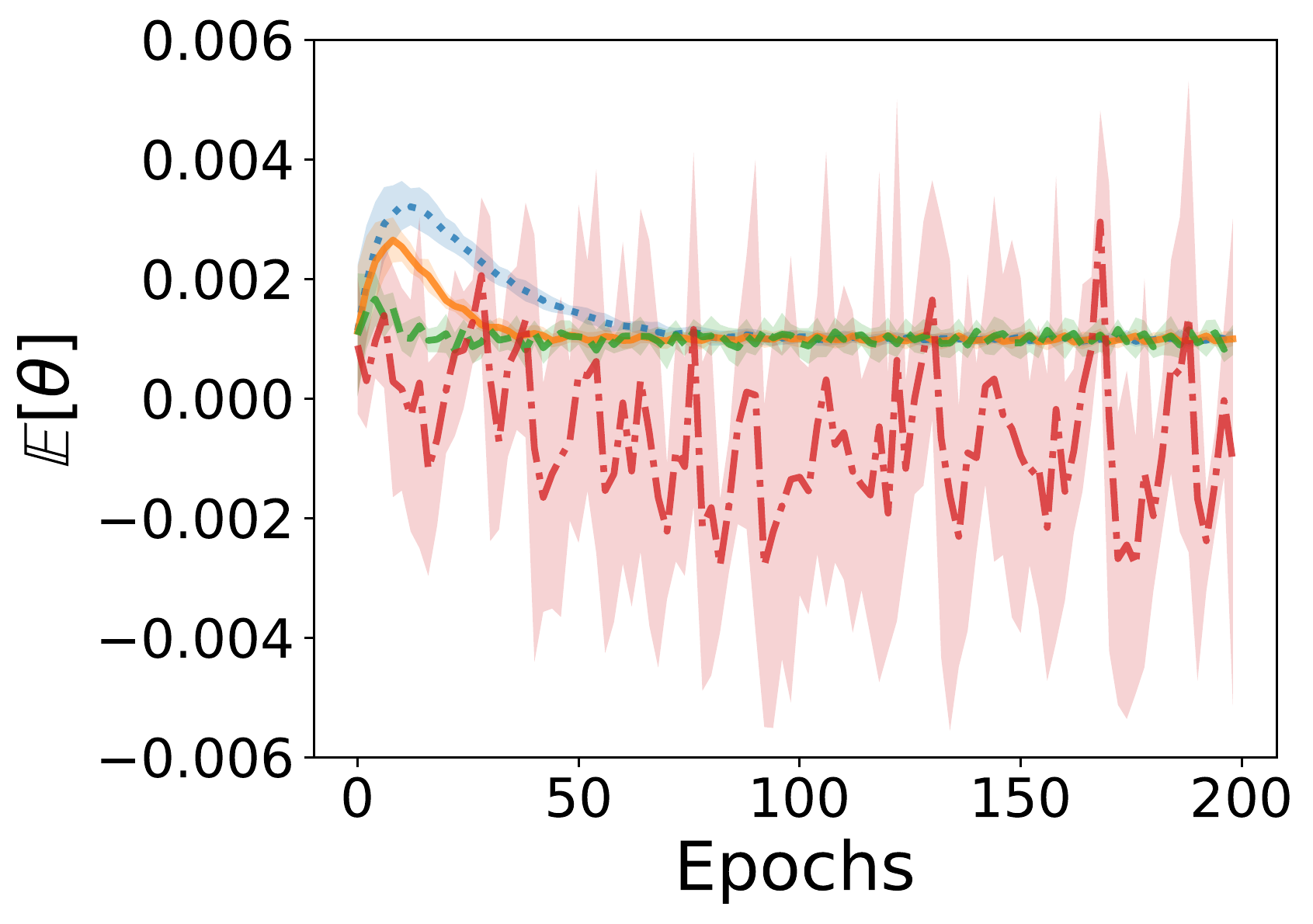}
  \caption{Hinge loss, average $\theta$.}
  \label{fig:100d_hinge_loss_avg_w_high}
\end{subfigure}

\begin{subfigure}{0.33\textwidth}
  \centering
  \includegraphics[width=\linewidth]{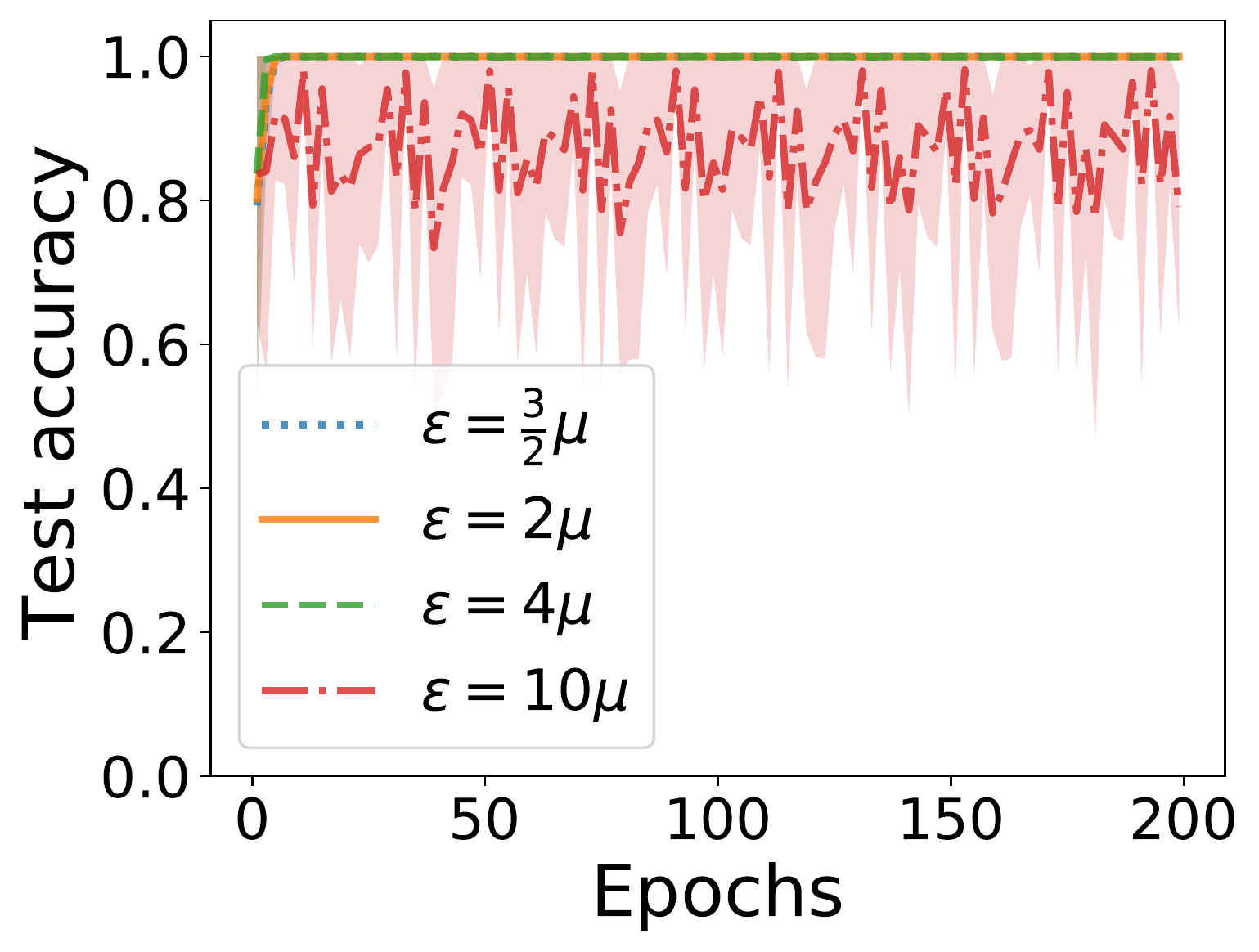}
  \caption{Linear loss, test accuracy.}
  \label{fig:100d_linear_loss_test_acc_high}
\end{subfigure}%
\begin{subfigure}{0.33\textwidth}
  \centering
  \includegraphics[width=\linewidth]{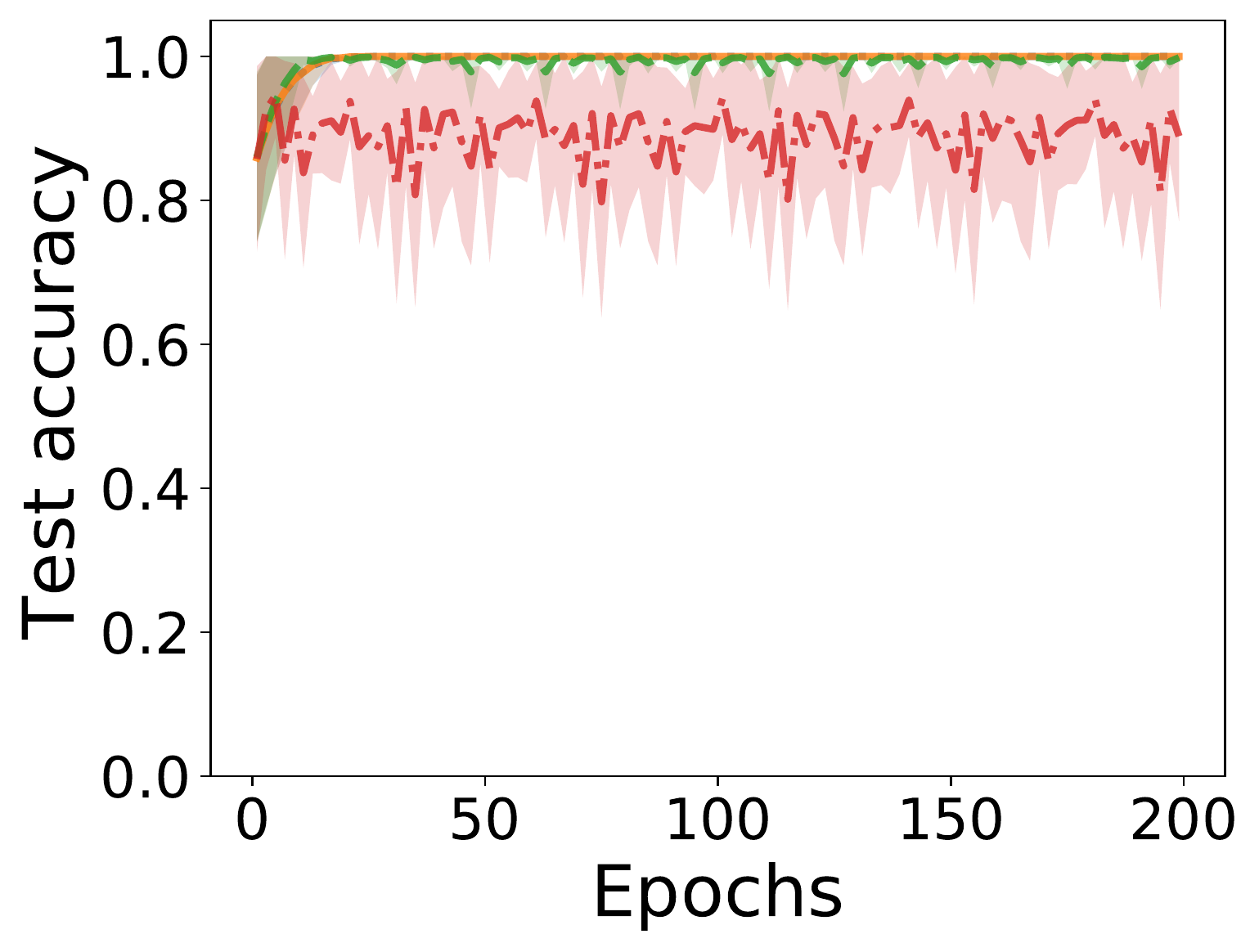}
  \caption{Cross-entropy loss, test accuracy.}
  \label{fig:100d_cross_entropy_loss_test_acc_high}
\end{subfigure}%
\begin{subfigure}{0.33\textwidth}
  \centering
  \includegraphics[width=\linewidth]{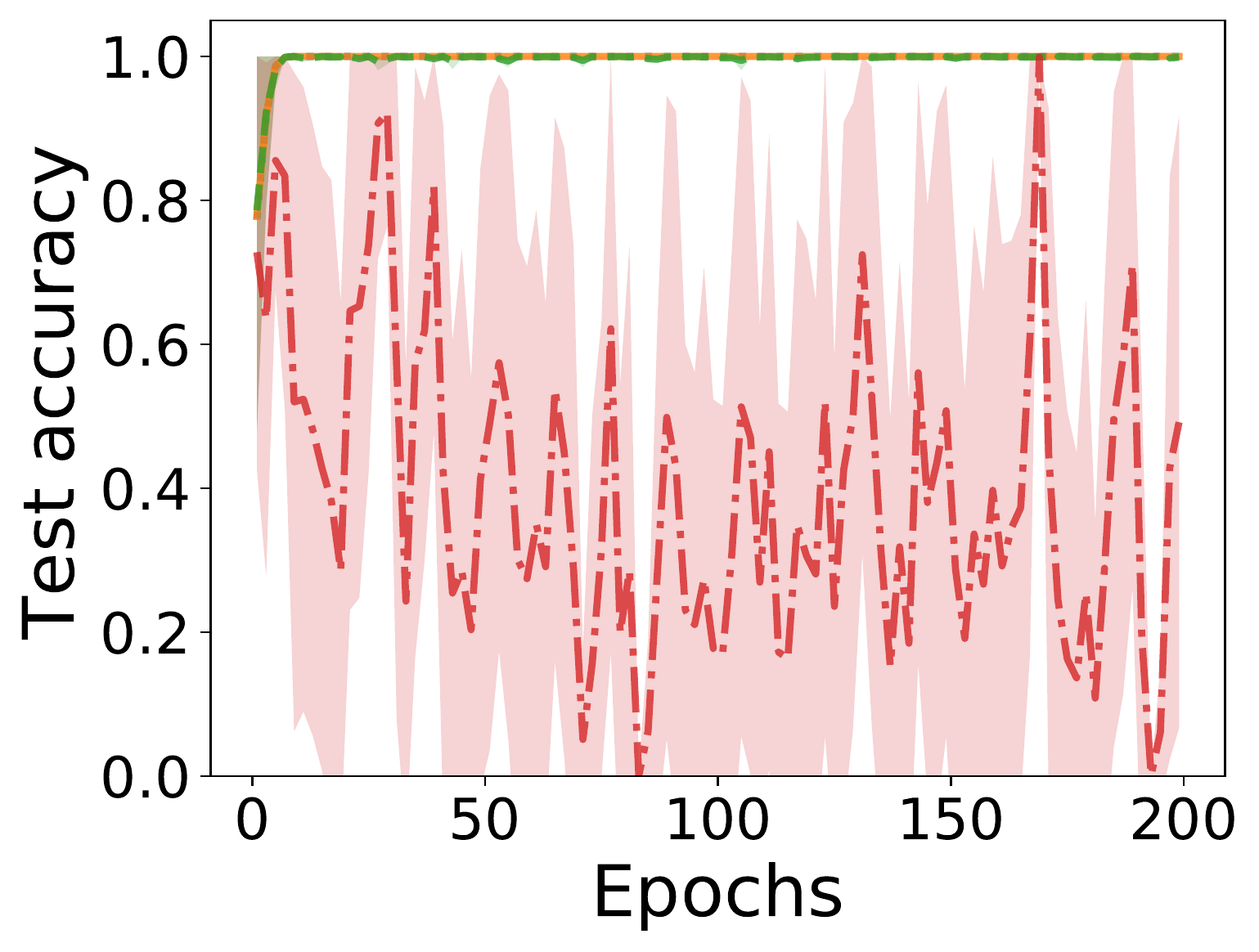}
  \caption{Hinge loss, test accuracy.}
  \label{fig:100d_hinge_loss_test_acc_high}
\end{subfigure}%
\caption{The average value of $\theta$ and test accuracy for a $100d$ Gaussian data model with $\mu=1$, $\sigma^2=1$, for $\epsilon > \mu$.}
\label{fig:100d_high}
\end{figure*}

We now give empirical results on the ability for a linear classifier to learn on the Gaussian data model under $\epsilon > \mu_j$. 
Firstly, we set $d=1$, $\mu=1$, $\sigma^2=1$, $\eta=0.001$, and set the number of gradient descent iterations to $100K$, sampling a new input from the Gaussian data model at each iteration. For $\epsilon \in [0, 20]$, we visualize the number of iterations during gradient descent where $\theta > 0$, implying an accurate classifier is learnt at these iterations, and the number of iterations at which $\theta < 0$.
\Cref{fig:1d_linear_loss} and \cref{fig:1d_cross_entropy_loss} show for a linear and cross-entropy loss function, respectively, if $0 < \epsilon < 20$, the number of occurrences where $\theta$ is positive during gradient descent exceed the number of negative occurrences.
Thus, even if $\epsilon$ is $20\times$ larger than $\mu$, it is more likely the final classifier has a small generalization error than a large generalization error. However, clearly the probability of learning a highly accurate classifier decreases as $\epsilon$ increases. As $\epsilon \rightarrow \infty$, it is less likely that consecutive gradient descent steps both have positive parameter values, and so the count of positive and negative parameter values during training $\rightarrow \nicefrac{\text{number of gradient descent steps}}{2}$. 
The test set accuracy for the linear loss (on $100K$ samples) at $\epsilon=0$ and $\epsilon=20$, is 84.12\% and 84.06\%, respectively. 
Note, that by~\cref{eq: bayes_error}, the maximum possible test set accuracy is $1-\frac{1}{2}\big(1 - \erf(\frac{1}{\sqrt{2}})\big)$, which is equal to 84.13\%.

Next, we empirically demonstrate the ability to learn with large $\epsilon$ on a higher dimensional Gaussian data model for the linear loss, cross-entropy loss, and hinge loss. For every $j\in [d]$, we set $\mu_j=\mu=1$, $\sigma^2=1$, $\eta=0.001$, and set $d=100$. 
We set both the number of training and test inputs to $100K$, and set the number of epochs to $200$.~\Cref{fig:100d_high} shows both the average value of $\theta$ and the test set accuracy throughout training, for values of $\epsilon$ greater than $\mu$. 
Firstly, we observe that for both the hinge and linear loss, the average value of $\theta$ oscillates around the learning rate, $\eta$. 
For larger $\epsilon$ values, the oscillations are greater in magnitude and occasionally fall below zero -- this is also true for the cross-entropy loss where the average value of $\theta$ oscillates around zero.
For $\epsilon = \frac{3}{2}\mu, 2\mu, 4\mu$, we achieve perfect test set accuracy after approximately 10 epochs, for each choice of loss function. 
For $\epsilon = 10\mu$, the linear loss and cross-entropy loss function results in an average test set accuracy of between 80-90\%.
The test set accuracy when using the hinge loss oscillates more dramatically, ranging from 0\% at approximately epoch 80, to 100\% at approximately epoch 175. 
These two values correspond to the extremes of the average value of $\theta$ in~\cref{fig:100d_hinge_loss_avg_w_high}. 

Thus, it is possible to achieve perfect test set accuracy even when the adversarial budget causes inputs to appear to be completely mislabeled during training. 
For completeness, we plot corresponding average values of $\theta$ and test set accuracy for $\epsilon$ values smaller than $\mu$ in \cref{sec:more_experiments_gaussian}. 
As one may expect, for $0\leq \epsilon < \mu$, the average value of $\theta$ monotonically increases, and the classifier achieves perfect test set accuracy.

The phenomenon of being able to learn under an adversary that can perturb an input by an arbitrary amount does not transfer to more complex datasets. 
In~\cref{fig:cifar10_test_acc}, we plot the accuracy of a ResNet-18 classifier~\citep{he2016deep} on the CIFAR-10 test set~\citep{krizhevsky2009learning} that has been adversarially trained using projected gradient descent~\citep{madry2017towards} with large values of $\epsilon$ ($\epsilon = \frac{32}{255}, \frac{127}{255}, \frac{255}{255}$). The learning rate was set to 0.1 and annealed to 0.01 and 0.001 at epochs 100 and 150, respectively, and we set the number of attack iterations in projected gradient descent to ten.
This corresponds to an adversary that can significantly distort the true value of each pixel during training. 
At $\epsilon =\frac{127}{255}, \frac{255}{255}$, the classifier is unable to learn and the test set accuracy does not grow above 10\% -- this corresponds to a classifier that guesses the label of an input uniformly at random. We plot examples of inputs at these values of $\epsilon$ in~\cref{sec:more_experiments_cifar}.

\begin{figure}[t]
  \centering
  \includegraphics[width=0.8\linewidth]{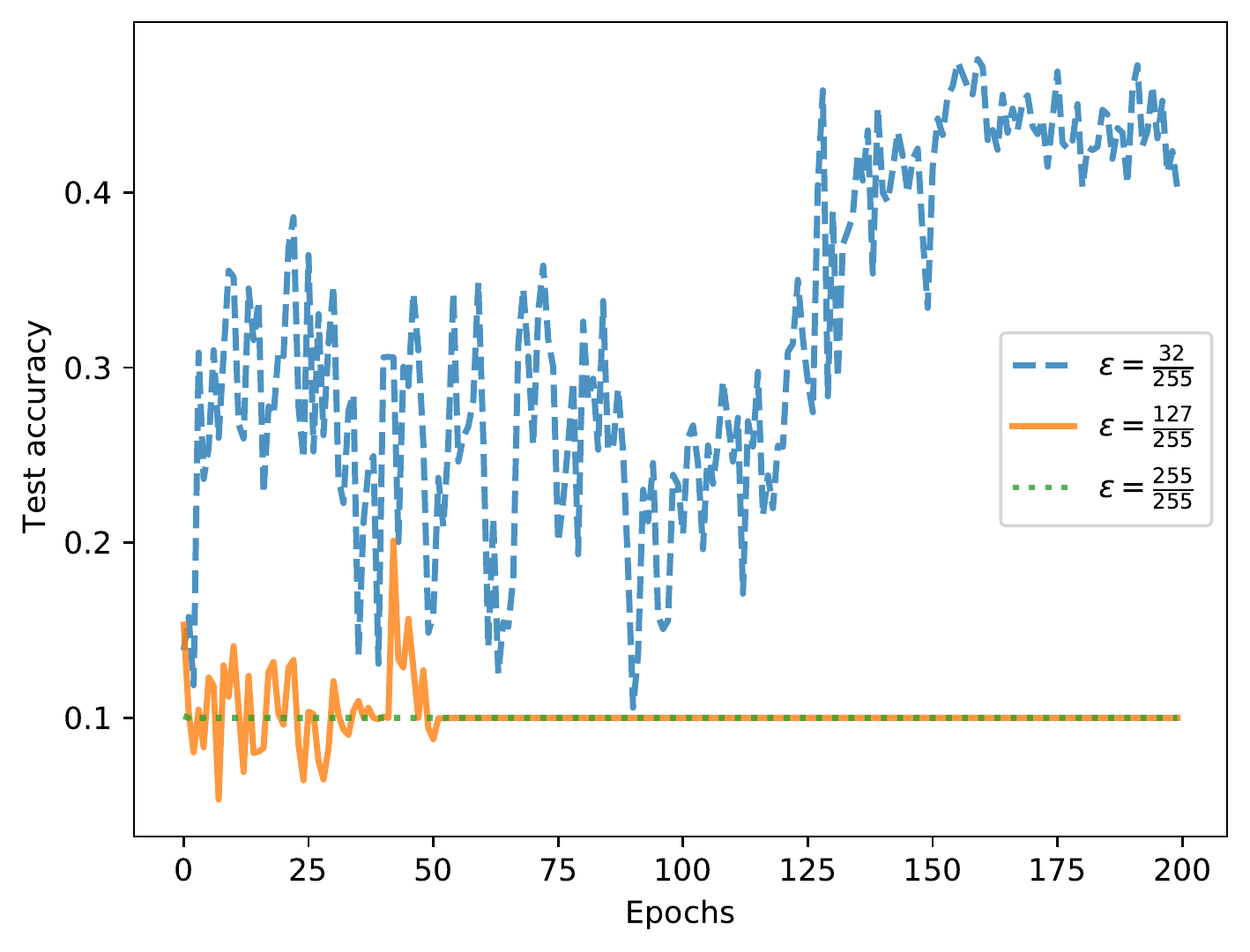}
  \caption{Test accuracy for an adversarially trained model on CIFAR-10.}
\label{fig:cifar10_test_acc}
\end{figure}

\section{Conclusion}
\label{sec:conclusion}

Due to the difficulty in characterizing complex high dimensional manifolds, it is attractive to defer to simple datasets in adversarial robustness research.
The hope being that insights made in these simple scenarios will be reflected in more `interesting' datasets and can be observed empirically. 
This work serves as a caution against such lines of reasoning; it is not always the case that theoretical
insights of adversarial robustness in simple data settings will transfer to other datasets.
In particular, we show that when the task is to learn a simple binary classification problem on Gaussian data, it is possible to achieve perfect standard test accuracy, when the classifier is trained with data perturbed by an arbitrarily strong adversary.
Thus, it is possible to achieve zero generalization error even if the data is almost entirely mislabeled. 
As one may expect, this is not a property that transfers to more realistic problems in computer vision such as classification of CIFAR-10 images.

\section*{Acknowledgements}

Jamie Hayes is funded by a Google PhD Fellowship in Machine Learning.

\bibliographystyle{icml2020}
\bibliography{references}

\newpage

\appendix
\section{More experimental results on Gaussian data}
\label{sec:more_experiments_gaussian}

Here, we plot analogous results to~\cref{fig:100d_high} for cases where the perturbation used in adversarial training, $\epsilon$, is smaller than $\mu$.
\Cref{fig:100d_medium} plots test accuracy and expected value of the parameter, $\theta$, for $\epsilon=\frac{3}{4}\mu$, $\epsilon=\frac{9}{10}\mu$, and $\epsilon=\mu$, while \cref{fig:100d_low} plots test accuracy and expected value of the parameter, $\theta$, for $\epsilon=0$ (standard training), $\epsilon=\frac{1}{10}\mu$, $\epsilon=\frac{1}{4}\mu$, and $\epsilon=\frac{1}{2}\mu$.

\begin{figure*}[htb!]
\centering
\begin{subfigure}{0.33\textwidth}
  \centering
  \includegraphics[width=\linewidth]{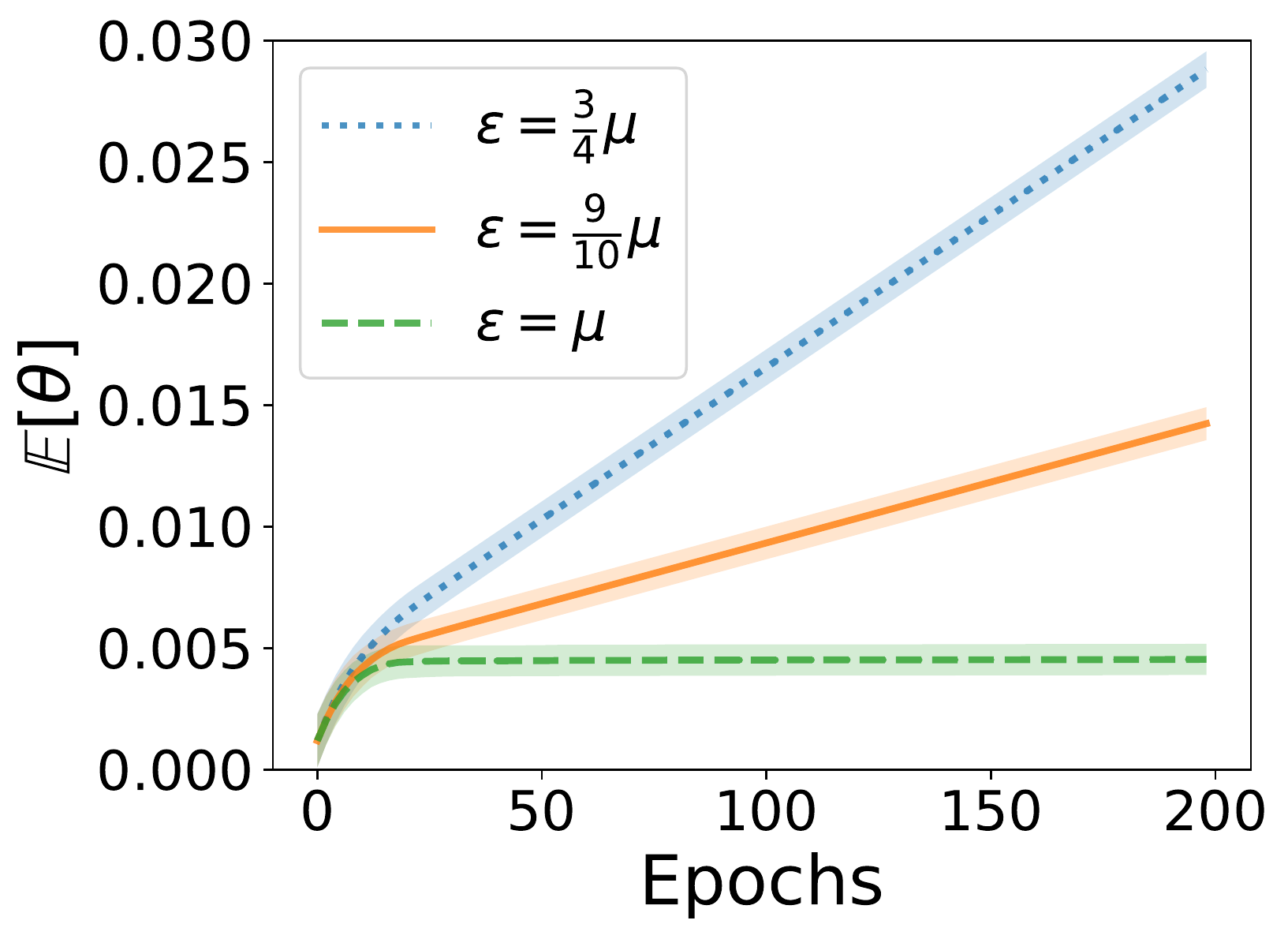}
  \caption{Linear loss, average $\theta$.}
  \label{fig:100d_linear_loss_avg_w_medium}
\end{subfigure}%
\begin{subfigure}{0.33\textwidth}
  \centering
  \includegraphics[width=\linewidth]{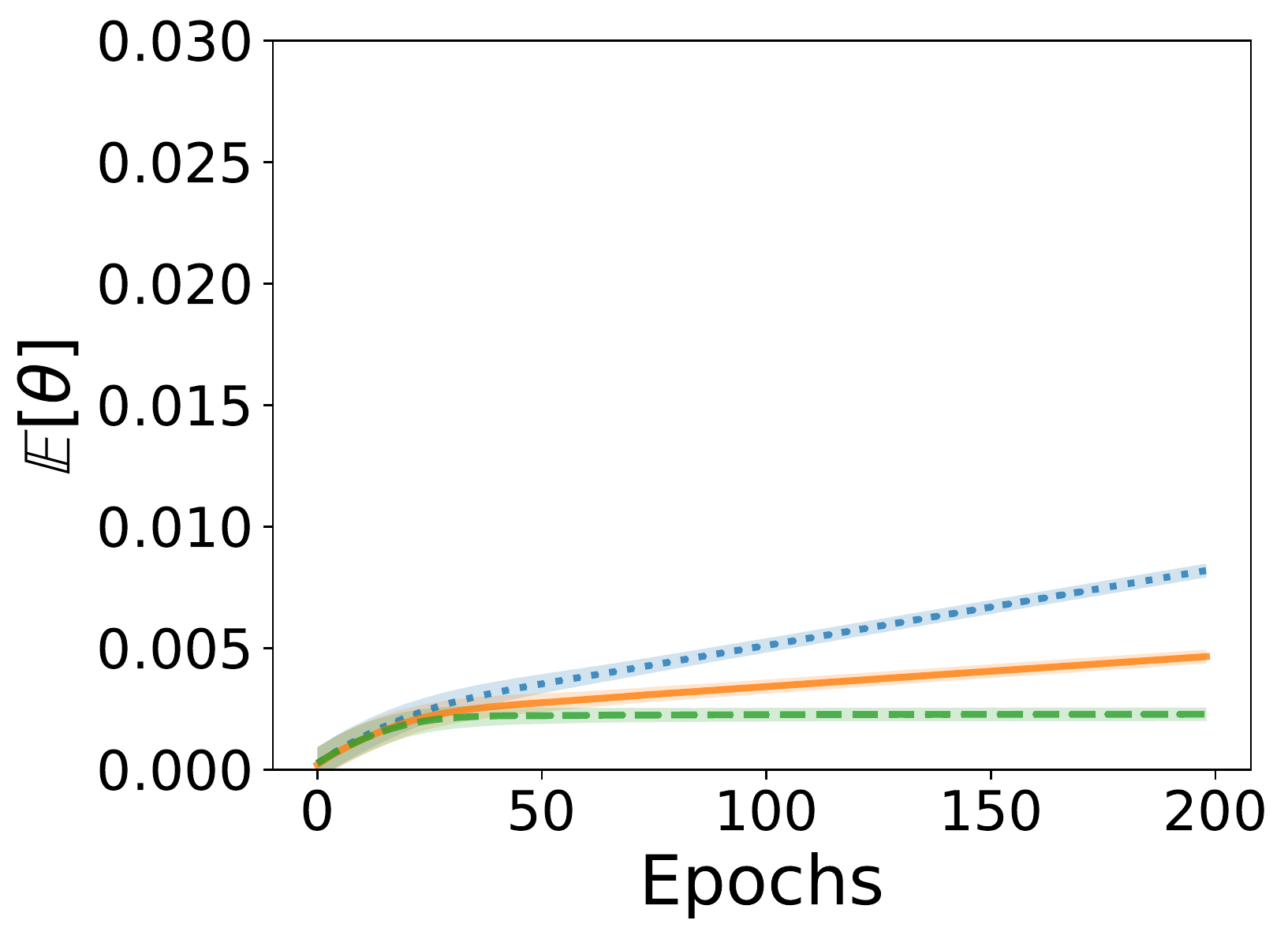}
  \caption{Cross-entropy loss, average $\theta$.}
  \label{fig:100d_cross_entropy_loss_avg_w_medium}
\end{subfigure}%
\begin{subfigure}{0.33\textwidth}
  \centering
  \includegraphics[width=\linewidth]{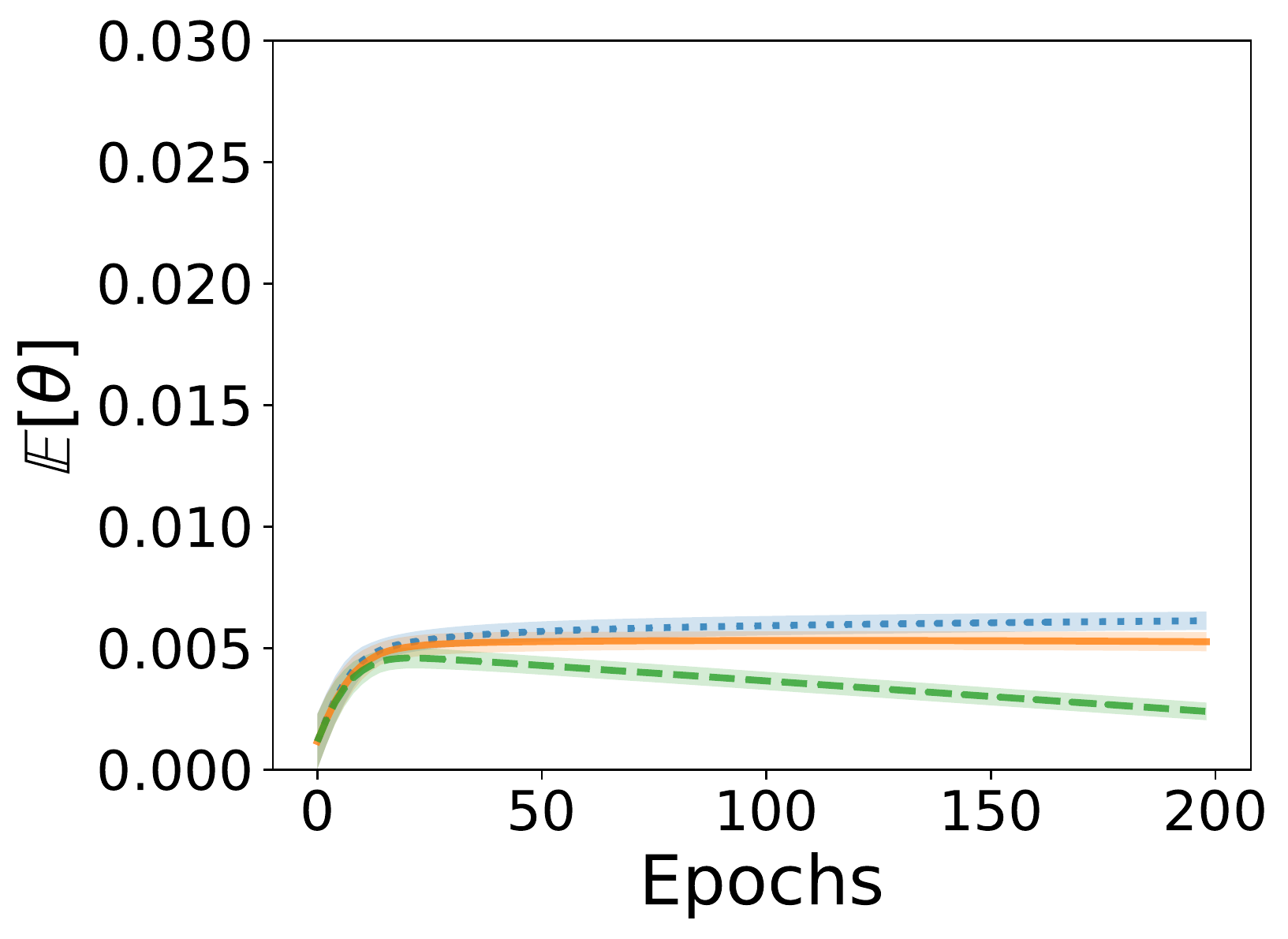}
  \caption{Hinge loss, average $\theta$.}
  \label{fig:100d_hinge_loss_avg_w_medium}
\end{subfigure}

\begin{subfigure}{0.33\textwidth}
  \centering
  \includegraphics[width=\linewidth]{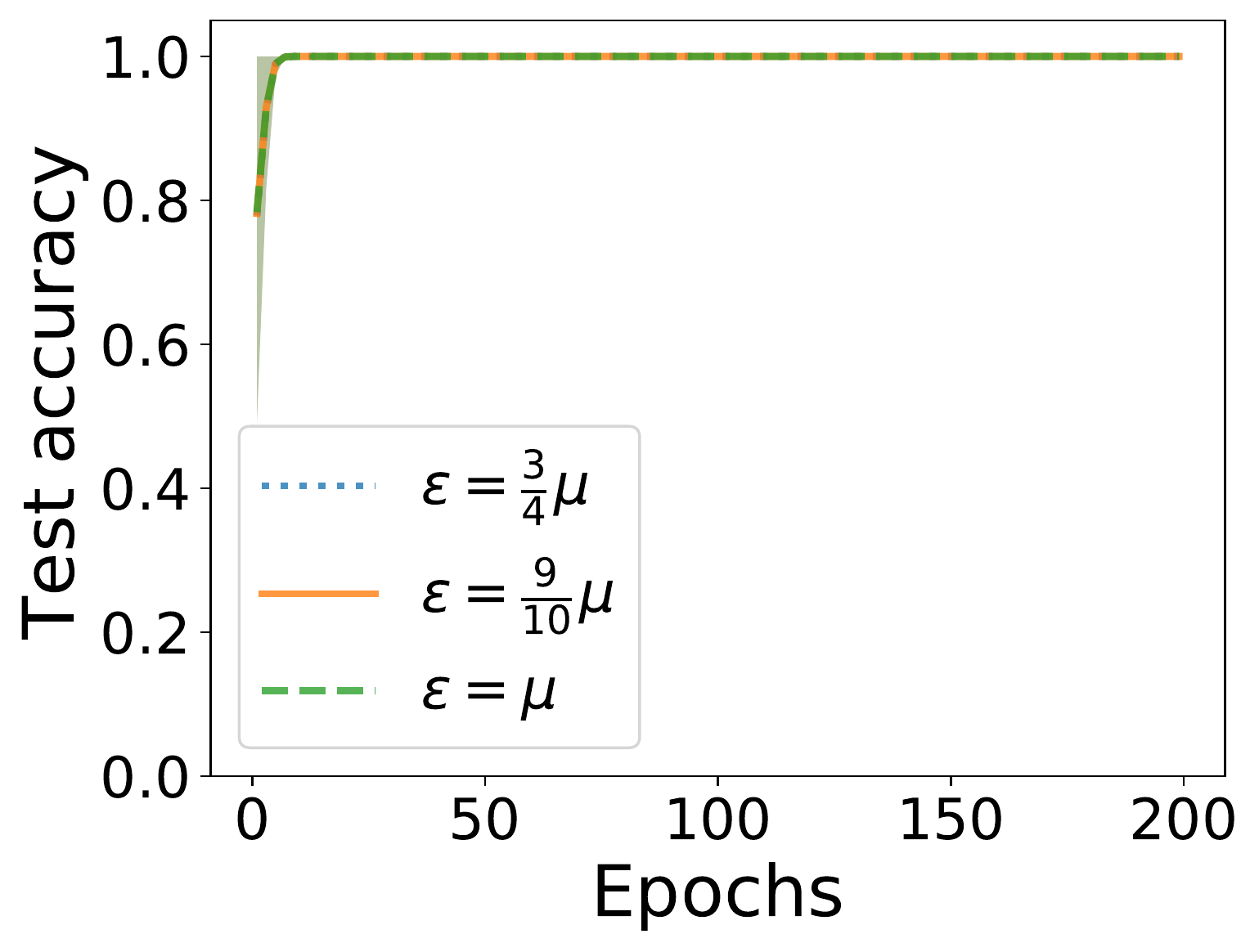}
  \caption{Linear loss, test accuracy.}
  \label{fig:100d_linear_loss_test_acc_medium}
\end{subfigure}%
\begin{subfigure}{0.33\textwidth}
  \centering
  \includegraphics[width=\linewidth]{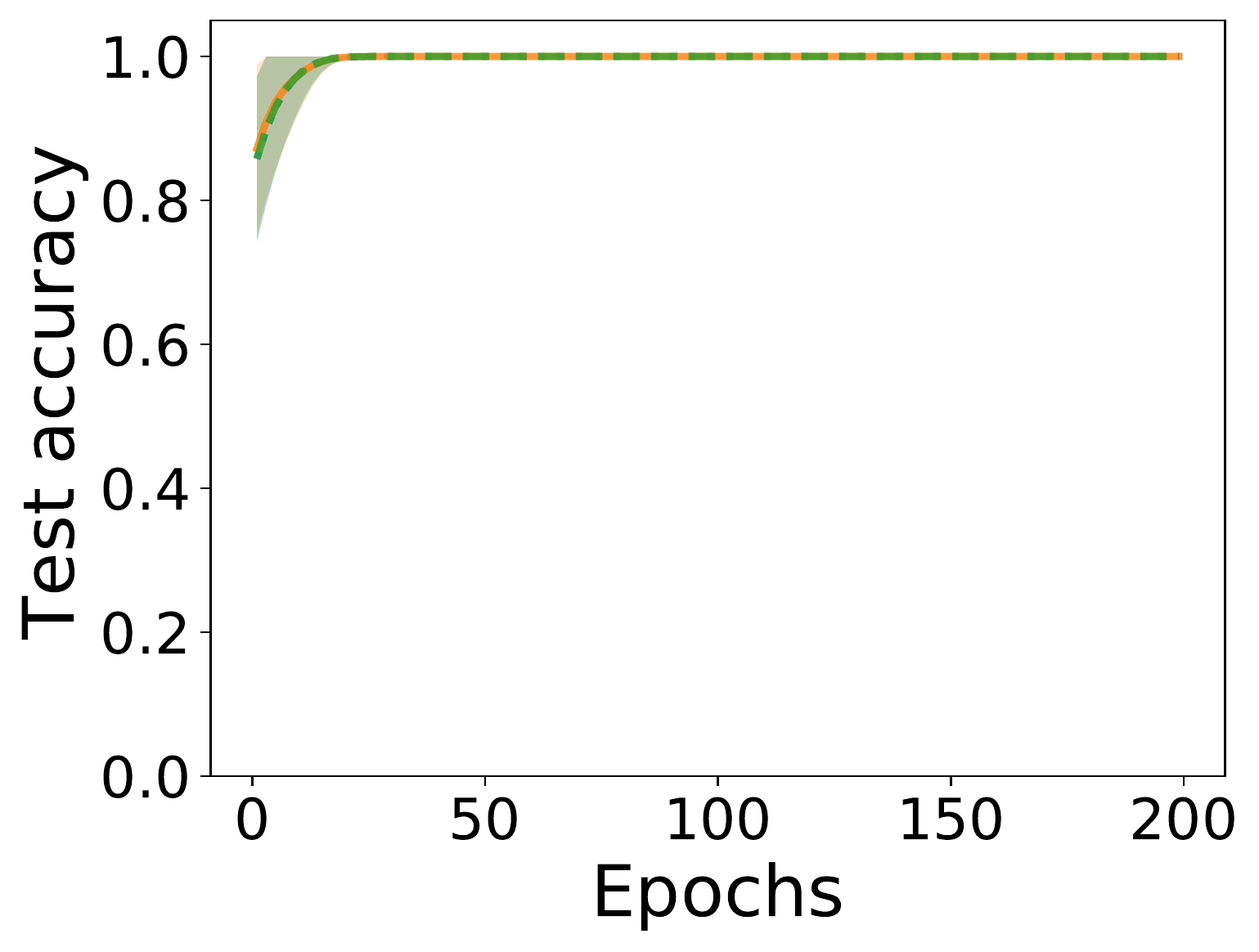}
  \caption{Cross-entropy loss, test accuracy.}
  \label{fig:100d_cross_entropy_loss_test_acc_medium}
\end{subfigure}%
\begin{subfigure}{0.33\textwidth}
  \centering
  \includegraphics[width=\linewidth]{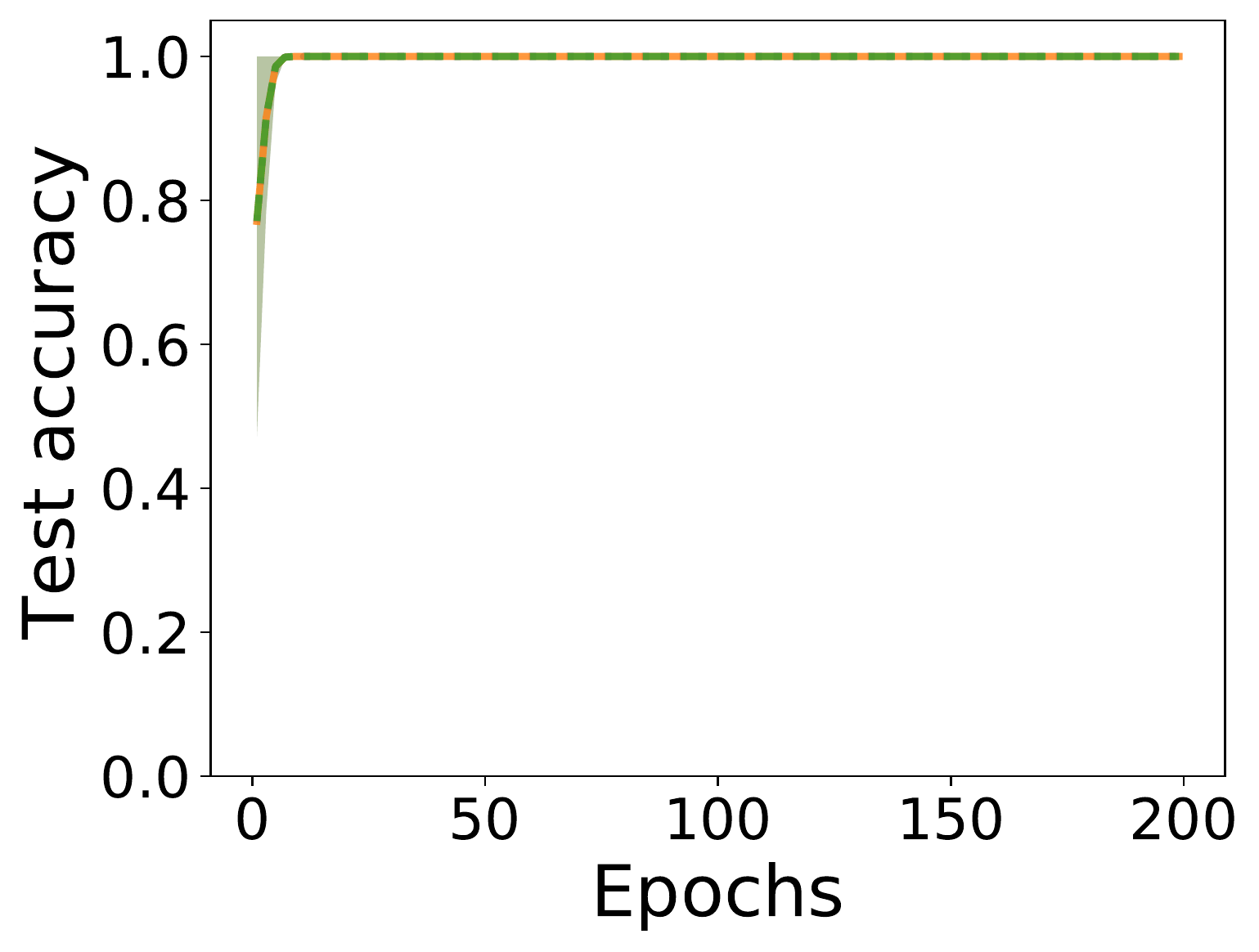}
  \caption{Hinge loss, test accuracy.}
  \label{fig:100d_hinge_loss_test_acc_medium}
\end{subfigure}%
\caption{The average value of $\theta$ and test accuracy for a $100d$ Gaussian data model with $\mu=1$, $\sigma^2=1$, for $\frac{\mu}{2} <\epsilon \leq \mu$.}
\label{fig:100d_medium}
\end{figure*}

\begin{figure*}[htb!]
\centering
\begin{subfigure}{0.33\textwidth}
  \centering
  \includegraphics[width=\linewidth]{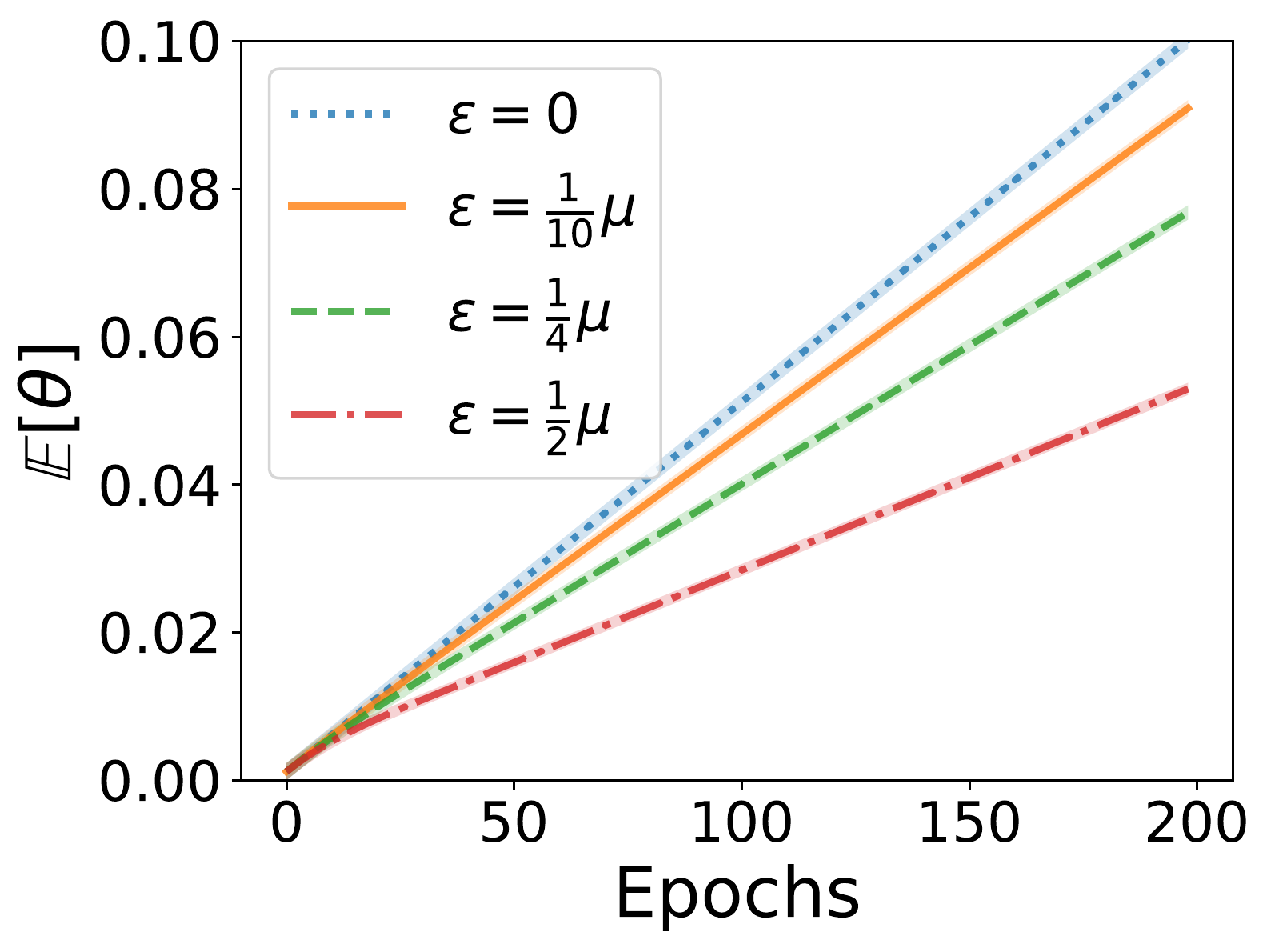}
  \caption{Linear loss, average $\theta$.}
  \label{fig:100d_linear_loss_avg_w_low}
\end{subfigure}%
\begin{subfigure}{0.33\textwidth}
  \centering
  \includegraphics[width=\linewidth]{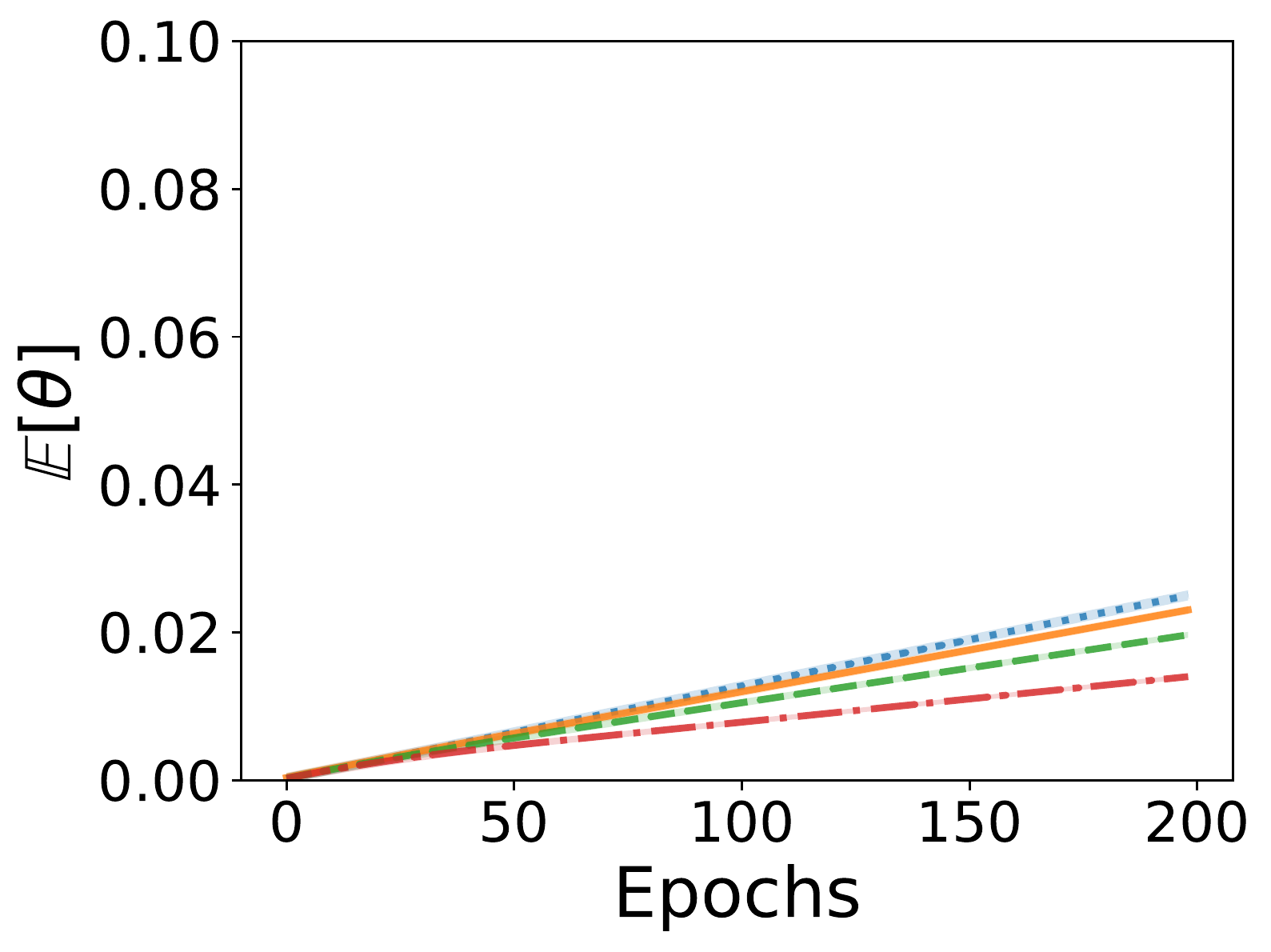}
  \caption{Cross-entropy loss, average $\theta$.}
  \label{fig:100d_cross_entropy_loss_avg_w_low}
\end{subfigure}%
\begin{subfigure}{0.33\textwidth}
  \centering
  \includegraphics[width=\linewidth]{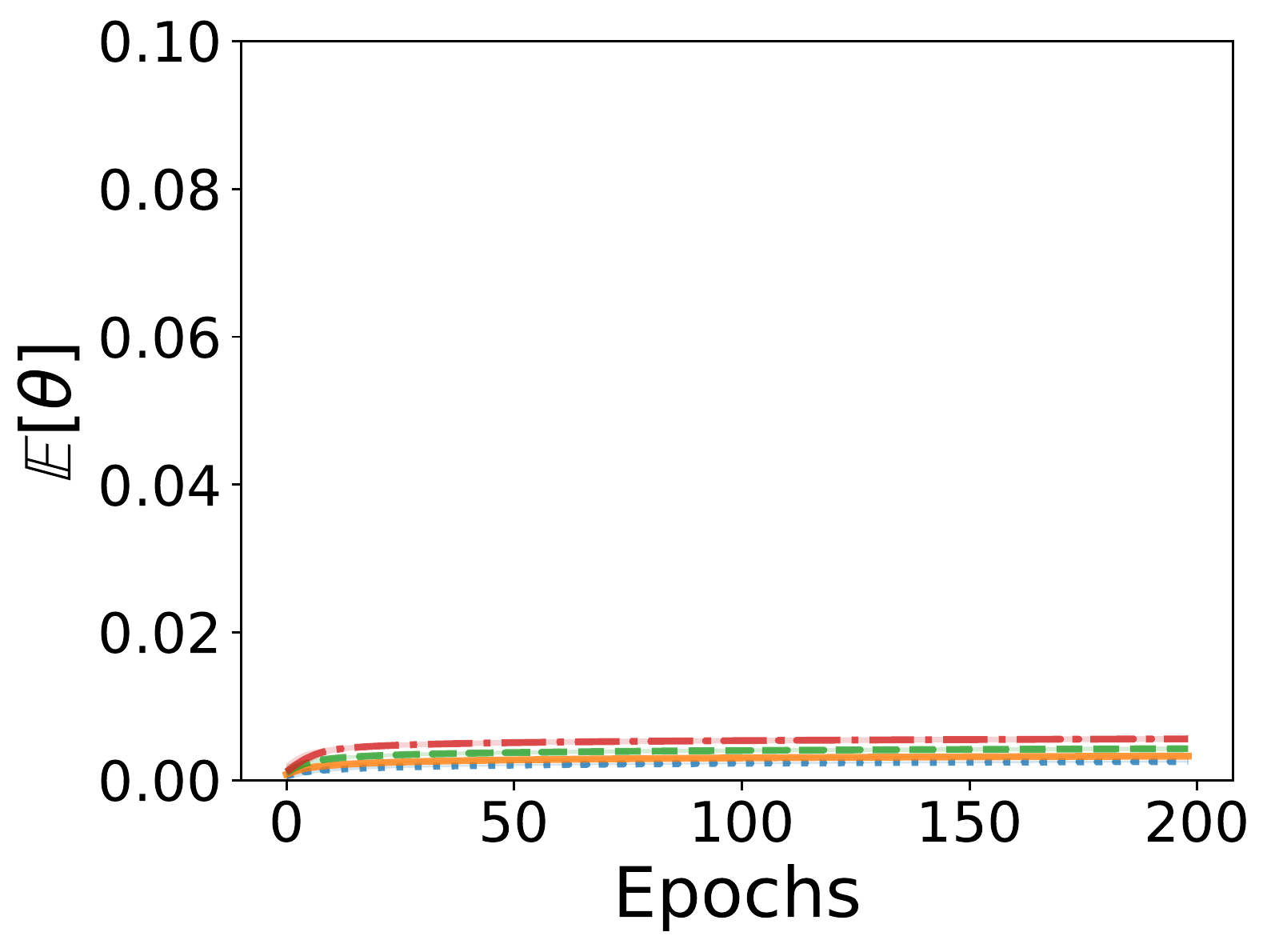}
  \caption{Hinge loss, average $\theta$.}
  \label{fig:100d_hinge_loss_avg_w_low}
\end{subfigure}

\begin{subfigure}{0.33\textwidth}
  \centering
  \includegraphics[width=\linewidth]{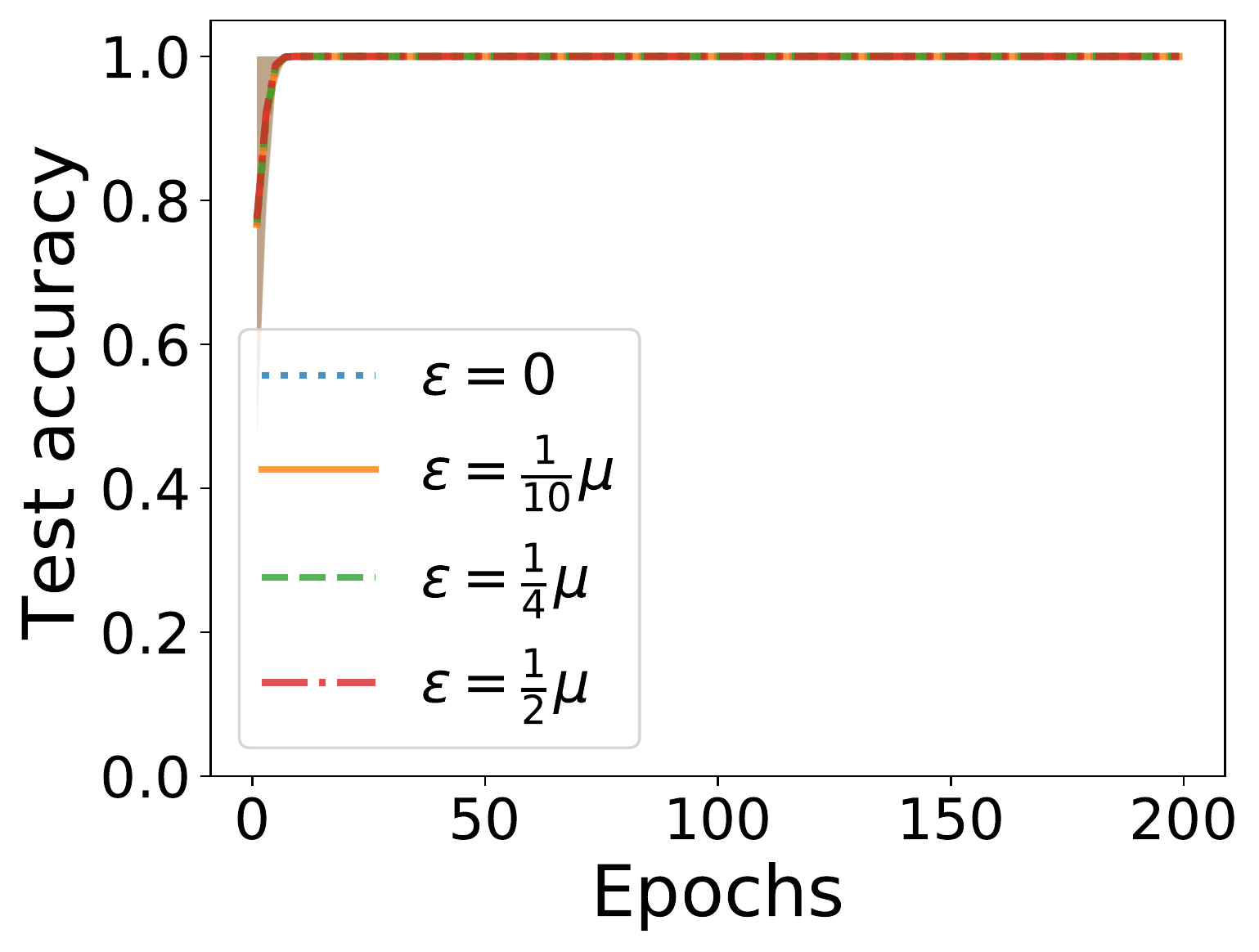}
  \caption{Linear loss, test accuracy.}
  \label{fig:100d_linear_loss_test_acc_low}
\end{subfigure}%
\begin{subfigure}{0.33\textwidth}
  \centering
  \includegraphics[width=\linewidth]{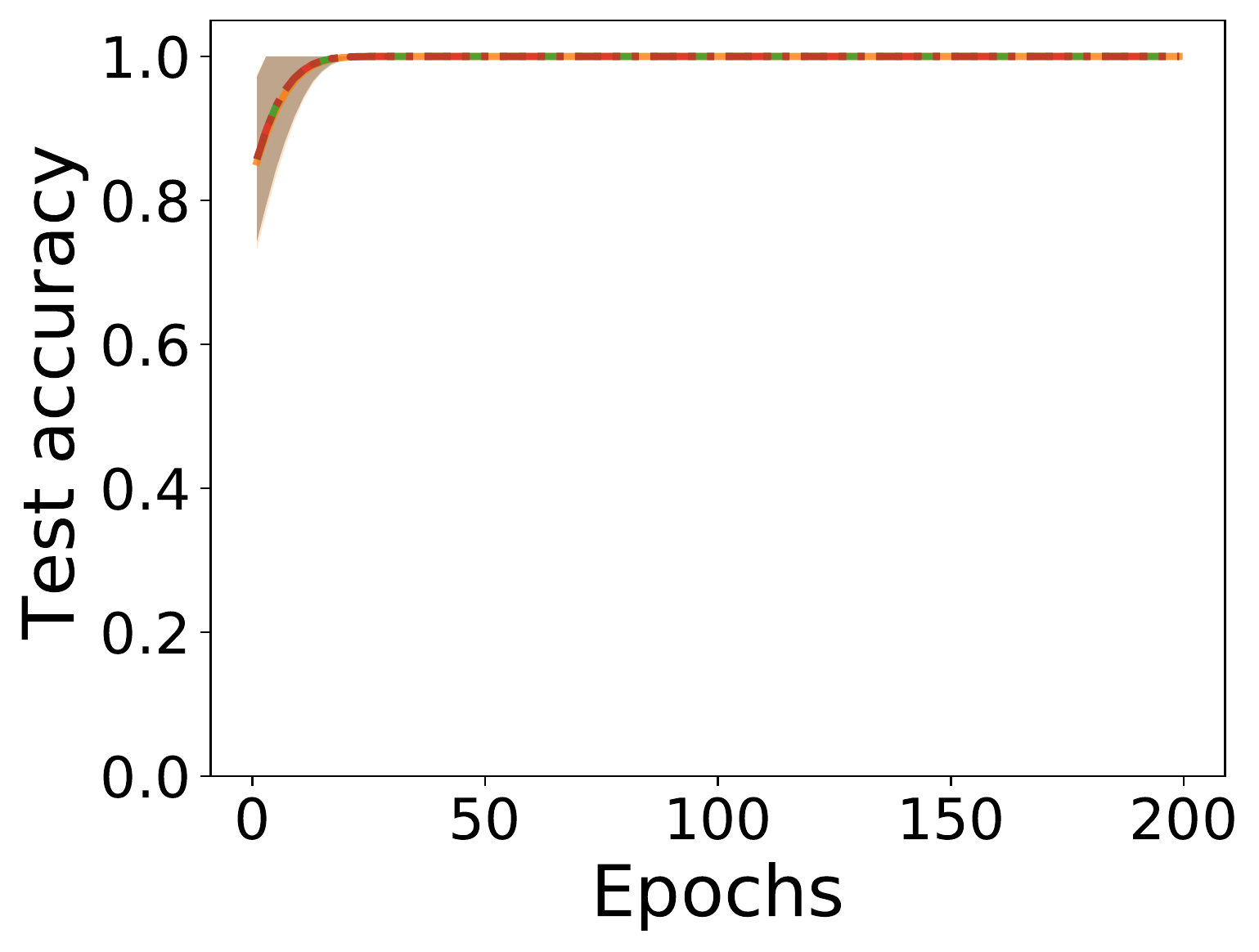}
  \caption{Cross-entropy loss, test accuracy.}
  \label{fig:100d_cross_entropy_loss_test_acc_low}
\end{subfigure}%
\begin{subfigure}{0.33\textwidth}
  \centering
  \includegraphics[width=\linewidth]{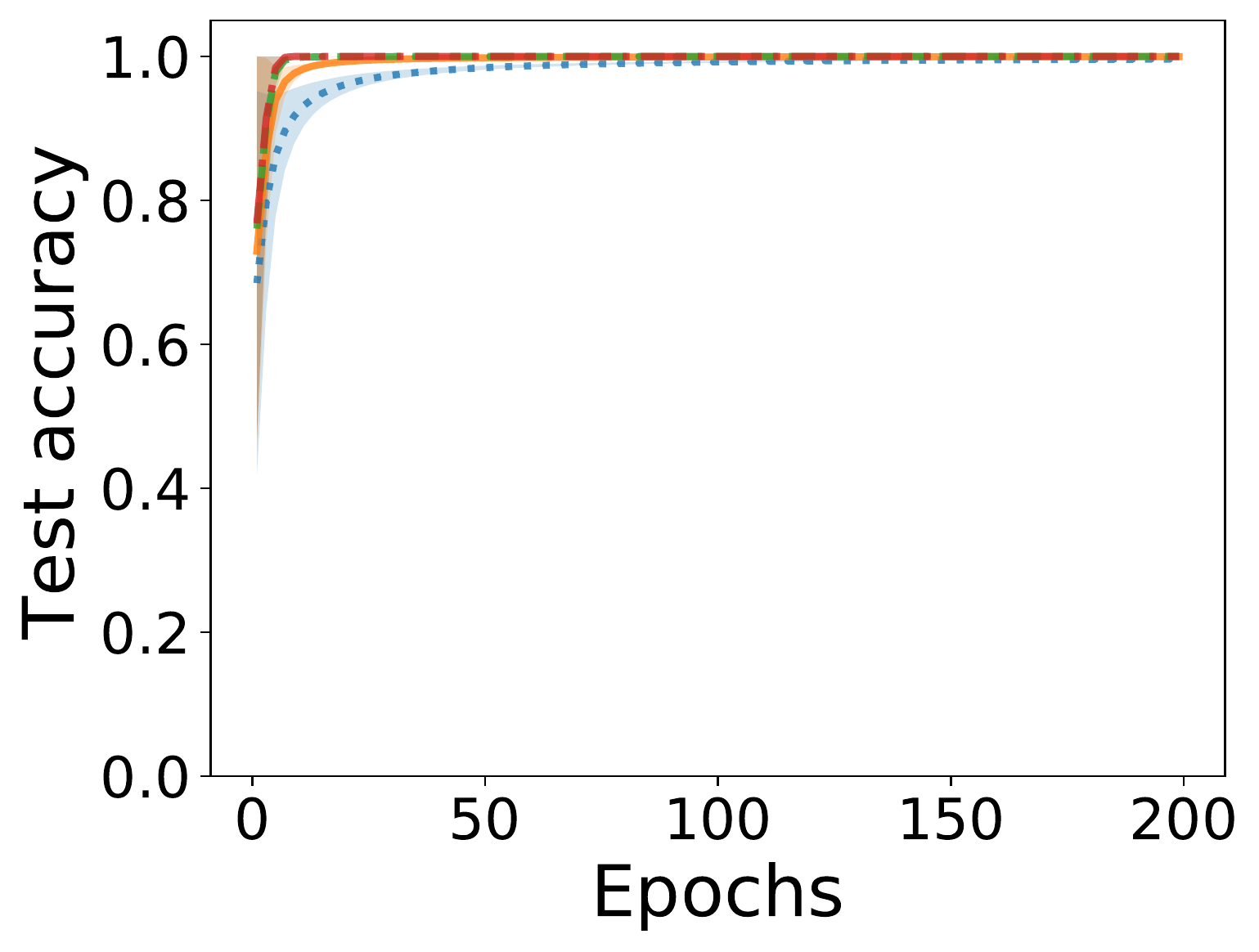}
  \caption{Hinge loss, test accuracy.}
  \label{fig:100d_hinge_loss_test_acc_low}
\end{subfigure}%
\caption{The average value of $\theta$ and test accuracy for a $100d$ Gaussian data model with $\mu=1$, $\sigma^2=1$, for $0 <\epsilon \leq \frac{\mu}{2}$.}
\label{fig:100d_low}
\end{figure*}

\clearpage

\section{More experimental results on CIFAR-10 data}
\label{sec:more_experiments_cifar}

In \cref{fig:cifar10_examples}, we plot qualitative samples of deformations caused by perturbations applied in adversarial training as detailed in the CIFAR-10 experiments in~\cref{sec: experiments}.

\begin{figure*}[htb!]
\centering
\begin{subfigure}{0.4\textwidth}
  \centering
  \includegraphics[width=0.95\linewidth]{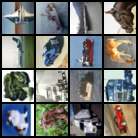}
  \caption{$\epsilon = 0$}
  \label{fig:cifar10_eps_0}
\end{subfigure}%
\begin{subfigure}{0.4\textwidth}
  \centering
  \includegraphics[width=0.95\linewidth]{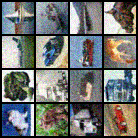}
  \caption{$\epsilon = \frac{32}{255}$}
  \label{fig:cifar10_eps_32}
\end{subfigure}

\begin{subfigure}{0.4\textwidth}
  \centering
  \includegraphics[width=0.95\linewidth]{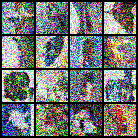}
  \caption{$\epsilon = \frac{127}{255}$}
  \label{fig:cifar10_eps_127}
\end{subfigure}%
\begin{subfigure}{0.4\textwidth}
  \centering
  \includegraphics[width=0.95\linewidth]{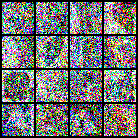}
  \caption{$\epsilon = \frac{255}{255}$}
  \label{fig:cifar10_eps_255}
\end{subfigure}%
\caption{Examples of CIFAR-10 images at various levels of distortion.}
\label{fig:cifar10_examples}
\end{figure*}

\section{A linear classifier with a linear loss cannot learn a shifted intercept}
\label{sec: problem2}

Here, we discuss further `weakness' of the linear loss function.
Let us re-define the Gaussian data model as follows. Let $(x,y)\in \mathbb{R}^d \times \{\pm 1\}$, where $y\in\{\pm 1\}$ is sampled uniformly at random, 
and $x\mid y=1 \sim \mathcal{N}(\mu_1, \sigma^2I)$,
$x\mid y=-1 \sim \mathcal{N}(\mu_2, \sigma^2I)$, where $0<\mu_2<\mu_1$ and $\sigma^2 >0$. 
This problem is not learnable when using $f_{\theta}(x) = \inp{\theta}{x} + b$ with loss function $\ell_1(f_{\theta}, y) = -y(\inp{\theta}{x} + b)$, while the problem becomes tractable by using the hinge loss, $\ell_2(f_{\theta}, y) = \max(0, -y(\inp{\theta}{x} + b))$. 
The problem stems from $\frac{\partial \ell_1}{\partial b} = -y$ and in expectation this is equal to zero since $y=\pm 1$ with probability $\nicefrac{1}{2}$. 
So the intercept parameter, $b$, will not be updated, and so if initialized by a standard Gaussian, will remain centered at zero.

\end{document}